\documentclass[letterpaper]{article} 
\usepackage{aaai25}  
\usepackage{times}  
\usepackage{helvet}  
\usepackage{courier}  
\usepackage[hyphens]{url}  
\usepackage{graphicx} 
\usepackage{natbib}  
\usepackage{caption} 
\usepackage{algorithm}
\usepackage{algorithmic}
\usepackage{multicol}
\usepackage{multirow}
\usepackage{subfigure}
\usepackage{subcaption}
\usepackage{amsmath,amssymb,amsthm}
\usepackage{amsthm}
\newtheorem{theorem}{Theorem}
\usepackage{booktabs}
\usepackage{xcolor}
\usepackage{dsfont}
\usepackage{appendix}
\usepackage{subfiles}
\usepackage{newfloat}
\usepackage{listings}
\DeclareCaptionStyle{ruled}{labelfont=normalfont,labelsep=colon,strut=off} 
\floatstyle{ruled}
\newfloat{listing}{tb}{lst}{}
\floatname{listing}{Listing}
\title{Boosting Fine-Grained Visual Anomaly Detection with \\ Coarse-Knowledge-Aware Adversarial Learning}
\author{
    Qingqing Fang\textsuperscript{\rm 1}, Qinliang Su\textsuperscript{\rm 1,2 \thanks{Corresponding author.}}, Wenxi Lv\textsuperscript{\rm 1}, Wenchao Xu\textsuperscript{\rm 3}, Jianxing Yu\textsuperscript{\rm 4,5}
}
\affiliations{
    \textsuperscript{\rm 1} School of Computer Science and Engineering, Sun Yat-sen University, Guangzhou, China\\
    \textsuperscript{\rm 2} Guangdong Key Laboratory of Big Data Analysis and Processing, Guangzhou, China \\
    \textsuperscript{\rm 3} Department of Computing, The Hong Kong Polytechnic University, Hong Kong SAR \\
    \textsuperscript{\rm 4} School of Artificial Intelligence, Sun Yat-sen University, Guangdong, China \\
    \textsuperscript{\rm 5} Pazhou Lab, Guangzhou, 510330, China \\
    
    \{fangqq3, lvwx8\}@mail2.sysu.edu.cn, \{suqliang, yujx26\}@mail.sysu.edu.cn, wenchao.xu@polyu.edu.hk
}

\usepackage{bibentry}
\begin{document}
\maketitle

\begin{abstract}

Many unsupervised visual anomaly detection methods train an auto-encoder to reconstruct normal samples and then leverage the reconstruction error map to detect and localize the anomalies. However, due to the powerful modeling and generalization ability of neural networks, some anomalies can also be well reconstructed, resulting in unsatisfactory detection and localization accuracy. In this paper, a small coarsely-labeled anomaly dataset is first collected. Then, a coarse-knowledge-aware adversarial learning method is developed to align the distribution of reconstructed features with that of normal features. The alignment can effectively suppress the auto-encoder's reconstruction ability on anomalies and thus improve the detection accuracy. Considering that anomalies often only occupy very small areas in anomalous images, a patch-level adversarial learning strategy is further developed. Although no patch-level anomalous information is available, we rigorously prove that by simply viewing any patch features from anomalous images as anomalies, the proposed knowledge-aware method can also align the distribution of reconstructed patch features with the normal ones. Experimental results on four medical datasets and two industrial datasets demonstrate the effectiveness of our method in improving the detection and localization performance.
\end{abstract}

\begin{links}
    \link{Code}{https://github.com/Faustinaqq/CKAAD}
\end{links}

\section{Introduction}
\label{sec: Introduction}
Visual anomaly detection (AD) aims to detect images that show significant deviation from normal ones. In practice, the deviation could happen at image-level or at the more fine-grained pixel-level, where anomalies only appear at certain local areas, while the remaining large region looks to be normal~\cite{yang2022visual}. Fine-grained anomalies happen a lot in real-world applications, such as the X-ray/CT images with nodules or other lesions in medical diagnosis~\cite{fernando2021deep}, images of products with defects like undesired scratches or small holes in industrial manufacturing {\it etc}~\cite{liu2024deep}. Because of the ubiquitousness of fine-grained anomalies and their intrinsic distinction from image-level anomalies, in this paper, we focus our attention on the fine-grained anomaly detection. In addition to detecting anomalies, some applications also require locating the anomalies, like automatically marking the defective areas~\cite{roth2022towards,liu2023simplenet}. Thus, it is of great value to develop methods that can simultaneously detect and locate anomalies accurately.

Considering the extreme diverseness of anomalies, existing methods \citep{ruff2018deep,akcay2019ganomaly,tack2020csi,roth2022patchcore,deng2022RD,reiss2023mean} mostly only assume the access to normal datasets. By using it to train a model that characterizes the normal patterns, anomalies can then be detected by assessing their conformity to it. The reconstruction-based method~\cite{akcay2019ganomaly,vasilev2020q,zavrtanik2021reconstruction,guo2023recontrast}, as one of the most widely-used fine-grained anomaly detection methods, is a concrete example of this idea, which first trains an auto-encoder on normal images and then uses the reconstruction error map to detect and localize anomalies. However, due to the strong modeling and generalization ability of neural networks, it is found that auto-encoders sometimes can also reconstruct the anomalous areas well due to the leakage of visual information from other parts of images, causing the reconstruction error map unreliable for anomaly detection and localization~\cite{you2022unified}.

The reason why the reconstruction-based methods cannot lead to competitive performance can be largely attributed to the only usage of normal samples when training the auto-encoder, which makes the decision boundary between normal and anomalous samples blurred~\cite{xia2022gan}. To boost the detection and localization accuracy, we propose to introduce a small dataset comprised of anomalous images. We argue that it is feasible to collect a small number of anomalies in many real-world applications~\cite{ruff2019deepsad,pang2021devnet,ding2022dra}. As exemplified in medical diagnosis and industrial manufacturing, there often exist some X-Ray images that have been confirmed to be problematic by veteran doctors, or images that are taken for products with defects. But compared with datasets in classification tasks, the anomaly datasets possess two unique characteristics: 1) {\emph{Coarseness}}: anomalies only occupy a small area of the image, but only the entire image is known to be anomalous; 2) {\emph{Incompleteness}}: the dataset only covers a tiny subset of all possible anomaly types due to the diverseness of anomalies. Thus, the dataset can be viewed as a kind of coarse anomaly knowledge, and our goal is to make use of it for better anomaly detection and localization. We note that some recent works have also proposed to collect an additional anomaly dataset and then use it to help detect anomalies like the Deep SAD~\cite{ruff2019deepsad}, PANDA-OE~\cite{reiss2021panda}, DevNet~\cite{pang2021devnet}, DRA~\cite{ding2022dra}, AA-BiGAN~\cite{tian2022aa-bigan}, PReNet~\cite{pang2023deep}, {\it etc}. However, these methods simply view an entire image in the anomaly dataset as anomalous, and rarely explicitly take the coarseness characteristics of anomalies into account. Thus, when they are applied to fine-grained anomaly detection problems, they may not make the best use of the knowledge hidden in the anomaly dataset, resulting in a performance that is not good enough. Moreover, all of these methods only consider how to improve anomaly detection accuracy, but rarely bear in mind the anomaly localization task that is also important in many scenarios.

In this work, we proposed to leverage the coarse anomaly knowledge to boost the anomaly detection and localization performance of feature-reconstruction-based methods. 
To this end,  a Coarse-Knowledge-Aware adversarial learning Anomaly Detection (CKAAD) strategy is developed to suppress the auto-encoder's reconstruction ability on anomalies by seeking to align the distribution of reconstructed features with that of normal features, where the features are extracted from a pre-trained ResNet in advance. Specifically, we first utilize the coarse anomalous knowledge to develop an energy-based discriminator, which is trained to assign low energy for normal features, but high energy for reconstructed features and anomalous features. The discriminator can effectively incentivize the auto-encoder to output normal features, while explicitly avoiding outputting anomalous features, thereby achieving the alignment with the distribution of normal features. Considering that anomalies often only occur in small areas in the anomalous images, a patch-level adversarial learning strategy is further developed. Because there are no patches that are known to be anomalous surely, we show that we can simply view all patches from an anomalous image as anomalies and then apply a similar adversarial learning strategy at image-level to guide the auto-encoder's reconstruction. We rigorously prove that the proposed patch-level adversarial learning strategy can also lead to the alignment between the reconstructed features and normal features, even though the patch-based anomalous knowledge is not completely correct. 
Experimental results on real-world anomaly datasets including four medical datasets and two industrial datasets demonstrate that by leveraging a small number of coarsely labeled anomalies, our proposed coarse-knowledge-aware method can boost the detection and localization performance remarkably.

\section{Related Work}

\paragraph{Unsupervised methods} Existing unsupervised methods only use normal samples. Among them, reconstruction-based methods commonly use auto-encoder~\cite{zavrtanik2021reconstruction,ZhangWC22}, variational auto-encoder~\cite{vasilev2020q} or generative adversarial network~\cite{schlegl2017unsupervised,akcay2019ganomaly} to reconstruct normal images, assuming that unseen anomalous images can not be reconstructed well, so the difference between the input and reconstructed images can be used to detect anomalies. Because these image-level methods hardly localize anomalous areas accurately, the feature reconstruction methods ~\cite{salehi2021multiresolution,deng2022RD,tien2023revisiting,guo2023recontrast} 
are further proposed. But whether reconstructing in the image or feature level, these reconstruction-based methods face the same problem that the anomalies are also restored well, decreasing the detection performance. Besides the reconstruction-based methods, one-class methods~\cite{ruff2018deep,yi2020patch,reiss2021panda,ijcai2022cap} map normal images or patches to a compact space, then determine
anomalies based on the distance. Embedding-based methods~\cite{tack2020csi,defard2021padim,ijcai2021mcl,reiss2023mean} discriminate anomalies according to the learned embedding features.
Other unsupervised methods~\cite{reiss2021panda,zavrtanik2021draem,li2021cutpaste,liu2023simplenet} utilize synthesized pseudo anomalies for directly predicting anomaly labels or masks, 
can result in poor detection performance on real anomalies.

\paragraph{Weakly-supervised methods} In recent years, some weakly-supervised methods have been proposed to enhance anomaly detection performance by utilizing observed partial anomalies.
Based on pre-trained models, PANDA-OE~\cite{reiss2021panda} directly employs classification models, but directly using a discriminator to find a boundary can result in misclassifications of unseen anomalies.
Instead of training a classifier,
Deep SAD~\cite{ruff2019deepsad} maps normal samples to a compact space, pushing anomalous samples away. 
DPLAN~\cite{pang2021dplan} guides the model to find anomaly samples and assigns higher anomaly scores through reinforcement learning.
GAN-based methods~\cite{kimura2020_wacv_adversarial,tian2022aa-bigan,tian2023leveraging,lvcontamination} avoid generating anomalous samples and get the anomaly score based on the norm of latent feature or the reconstruction error.
DevNet~\cite{pang2021devnet} and DRA~\cite{ding2022dra} use deviation loss or score heads to separate anomalous patches from normal ones.
PReNet~\cite{pang2023deep} detects anomalies based on pairwise anomaly scores learned by three kinds of instance pairs.
However, these methods directly enlarging the distinction of anomaly scores between normal and anomalous samples, are prone to over-fit on these collected samples. Besides, most of them primarily use these incomplete anomalies for detection, rarely leveraging such weakly supervised information to achieve accurate localization performance.

   

\section{Preliminaries on Feature Reconstruction-Based Anomaly Detection}
Due to the diverseness of anomalies, current anomaly detection methods are generally established on the assumption of access to a normal dataset ${\mathcal{X}}^+ = \left\{x_1^+, x_2^+, \cdots, x_n^+\right\}$. By seeking to characterize the normal patterns in normal samples as accurately as possible, anomalies can be identified by checking whether they conform to the normal patterns. Regarding methods following this approach, the most widely used one is the reconstruction-based method, which basically trains an auto-encoder (AE) on normal samples and then uses the reconstruction error to detect anomalies. 

For image data, to yield better performance, it is often suggested to reconstruct feature maps $\{F^\ell\}_{\ell=1}^L$ that are extracted from pre-trained ResNet ~\cite{he2016resnet}, instead of reconstructing original images, with $F^\ell$ and $L$ denoting the feature map from the $\ell$-th block of ResNet and the total number of blocks. Since the feature maps from different blocks are of different sizes, when implementing the encoder, we first down-sample maps to the same size as the smallest one by a sequence of $3\times 3$ convolutional operators and then concatenate them along the channel dimension. The concatenated maps are then further fed into a ResNet block to produce the final latent representation $z$, as shown in Figure \ref{fig:model}. The whole procedure can be described as $z = {\mathcal{E}}_\theta(\{{F}^s\}_{s\in {\mathcal{S}}})$, where ${\mathcal{S}}\subset \{1, 2, \cdots, L\}$ represents a subset of $L$ ResNet blocks that are chosen to participate in reconstruction. Given the latent code $z$, residual deconvolution operations are then employed to up-sample it stage by stage to recover the feature maps of different sizes $\tilde{F}^s$ for ${s\in {\mathcal{S}}}$, that is, $\{\tilde{F}^s\}_{s\in {\mathcal{S}}} = {\mathcal{D}}_\theta(z)$. By viewing the original feature maps $\left\{F^{s}\right\}_{s\in {\mathcal{S}}}$ as input and the recovered maps $\{\tilde{F}^s\}_{s\in {\mathcal{S}}}$ as output, the encoder ${\mathcal{E}}_\theta(\cdot)$ and decoder ${\mathcal{D}}_\theta(\cdot)$ can be jointly written as
\begin{align}
     \quad \{\tilde{F}^{s}\}_{s\in {\mathcal{S}}} = {G}_\theta\left(\left\{F^{s}\right\}_{s\in {\mathcal{S}}}\right),
\end{align}
where $G_\theta(\cdot) = {\mathcal{D}}_\theta({\mathcal{E}}_\theta(\cdot))$.

The auto-encoder $G_\theta(\cdot)$ is trained to reduce the difference between $F^s$ and $\tilde F^s$ for $s\in {\mathcal{S}}$ on all training normal samples by minimizing the loss
\begin{align}\label{eq:gcd}
L_{rec}^+\!=\!\frac{1}{|{\mathcal{X}}^+|}\!\!\! \sum_{x\in{\mathcal{X}}^+}  \!\! \sum_{s \in {\mathcal{S}}}\! \left(1\!-\!\cos \! \left(\!\text{flatten}(\!\tilde F^s\!), \text{flatten}(\!F^s\!)\right) \!\! \right)
\end{align}
where $\text{flatten}(\cdot)$ and $\cos(\cdot, \cdot)$ mean flattening the tensor into one-dimensional vector and the cosine similarity. Since the loss is trained only on normal samples, the auto-encoder tends to reconstruct the input normal images well, but poorly for anomalous ones. That is why the error maps between $F^s$ and $\tilde F^s$ can be employed to detect and localize the anomalies. However, due to the generalization ability of neural networks, it is found that some anomalous images can also be well reconstructed, making the error map not always reliable in detecting and localizing the anomalies.

\section{Boosting Reconstruction-Based Anomaly Detection via Coarse-Knowledge-Aware Adversarial Learning}

\begin{figure*}
    \centering
    \includegraphics[width=0.85\linewidth]{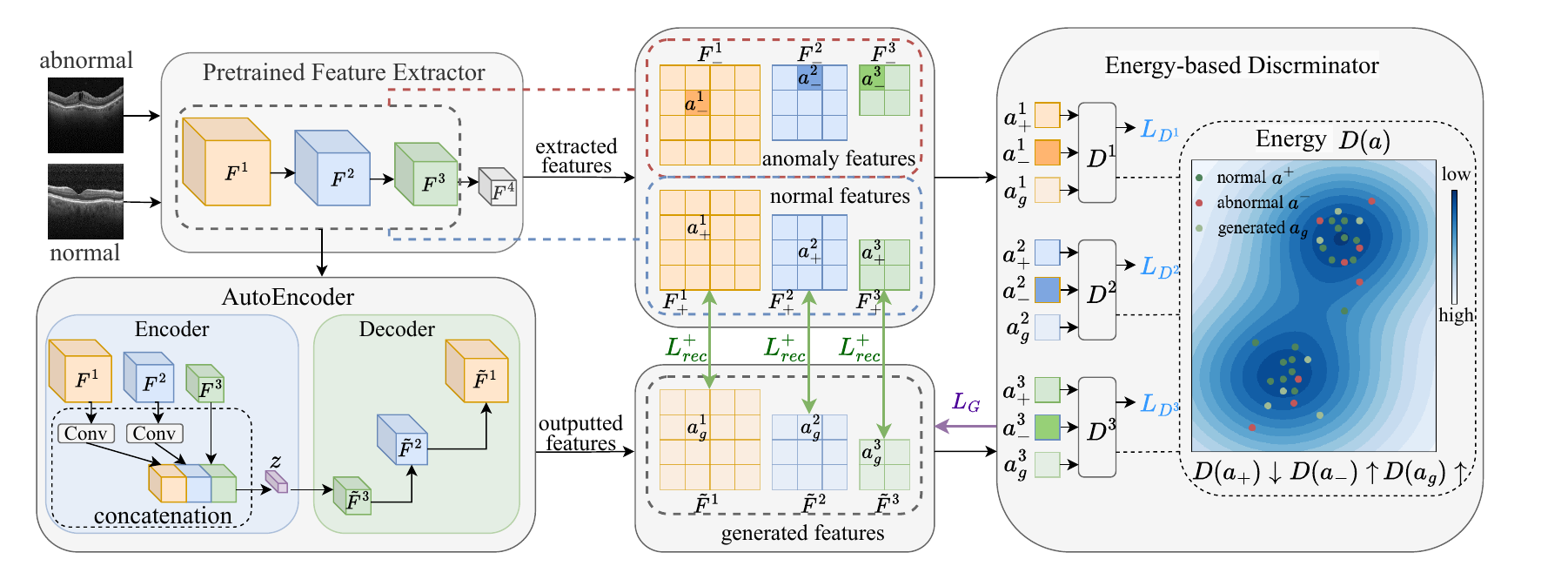}
    \caption{Framework of our proposed CKAAD. With the normal and anomalous feature maps extracted from pre-trained ResNet as well as the generated feature maps output from the auto-encoder, the energy-based discriminator $D^\ell$ is trained to distinguish between normal, anomaly, and generated patch features output from the $l$-th layer of ResNet. The auto-encoder is then incentivized to always output normal features regardless of the input types.}
    \label{fig:model}
\end{figure*}

To improve the detection and localization accuracy, in addition to the normal dataset ${\mathcal{X}}^+$, we further assume the availability of a small anomaly dataset
\begin{align}
    {\mathcal{X}}^- = \left\{x_1^-, x_2^-, \cdots, x_m^-\right\}.
\end{align}
As elaborated in the Introduction, due to the diverseness of anomalies, the collected anomalies in ${\mathcal{X}}^-$ may only cover a very tiny number of all possible anomaly types. Moreover, due to the annotating cost, we assume we only know the images in ${\mathcal{X}}^-$ are anomalous, but have no knowledge of the exact location of the anomalies. This paper aims to leverage this coarse anomalous knowledge to better detect and locate anomalies.

One of the most direct ways to leverage the dataset ${\mathcal{X}}^-$ is to minimize the following loss
\begin{align} \label{eq:L_rec}
    L_{rec} = L^+_{rec} - \lambda L^-_{rec},
\end{align}
where $L^-_{rec}$ denotes the reconstruction error computed over samples in ${\mathcal{X}}^-$; and $\lambda \in {\mathbb{R}}^+$ is a positive weighting factor. However, the loss $L_{rec}$ only drives the auto-encoder to output feature maps $\tilde F^s$ that look different from the input $F^s$ for anomalous samples, but never restricting how the difference should look like. Because anomalies often only occupy a small area in anomalous images in many real-world applications, simply enlarging the reconstruction error of anomalous samples has the risk of turning normal areas into anomalies, deteriorating the detection and localization performance. To enable the auto-encoder to detect and localize the anomalies accurately, we argue that we should encourage it to possess the following two properties: 
\begin{enumerate}
    \item[P1)] The auto-encoder should manage to transform an anomalous image to its normal counterpart, {\it e.g.}, transforming an anomalous image with a small hole into one that has the hole removed.
    \item[P2)] The reconstruction quality of normal samples should not be compromised too much by the introduced anomalies. 
\end{enumerate}
If the two appealing properties hold, the reconstruction error map between input and output features could be used to detect and localize the anomalies more reliably.

\subsection{Aligning the Distributions of Reconstructed Features with Normal Features}

To achieve the appealing properties above, we propose to view the auto-encoder as a generator $G_\theta(\cdot)$. For the simplicity of presentation, we only present how to process one feature map $F^s\in \{F^s\}_{s\in {\mathcal{S}}}$, and denote the chosen map $F^s$ as $F$ for conciseness, while the remaining feature maps can be processed similarly. Since we now have two datasets available, ${\mathcal{X}}^+$ and ${\mathcal{X}}^-$, we can use them to construct a distribution regarding input feature maps as
\begin{align}
    \mathbb{P}(F) = \alpha\mathbb{P}^+(F) + (1-\alpha)\mathbb{P}^-(F),
    \label{eq:p(F)}
\end{align}
where $\mathbb{P}^+(F)$ and $\mathbb{P}^-(F)$ represent feature map distributions of normal samples $x^+ \in {\mathcal{X}}^+$ and anomalous samples $x^- \in {\mathcal{X}}^-$, respectively; $\alpha \in [0, 1]$ is a parameter used to control the mixture ratio of the two distributions. For any input feature $F$ drawn from the distribution $\mathbb{P}(F)$, the generator (auto-encoder) will output a feature map $G_\theta(F)$. The distribution of the generated feature map $G_\theta(F)$ can be represented as
\begin{align}
\mathbb{P}_g(F) = \alpha\mathbb{P}_g^+(F) + (1-\alpha)\mathbb{P}_g^-(F),
\end{align}
where $\mathbb{P}_g^+(F)$ and $\mathbb{P}_g^-(F)$ denote the distributions of generated (reconstructed) feature map $G_\theta(F)$ when the input $F$ is drawn from $\mathbb{P}^+(F)$ and $\mathbb{P}^-(F)$, respectively. If the distribution of $\mathbb{P}_g(F)$ is forced to align with the distribution of normal features $\mathbb{P}^+(F)$, it amounts to driving $\mathbb{P}^-_g(F)$ towards $\mathbb{P}^+(F)$, that is, having the reconstructed features of anomalous samples to look similar to normal ones. In this way, the auto-encoder transforms an anomalous sample into a normal one effectively, thereby partially achieving the appealing property P1.

\paragraph{Image-Level Alignment}  To align the generative distribution $\mathbb{P}_g(F)$ with normal feature distribution $\mathbb{P}^+(F)$, a generative adversarial network (GAN)~\cite{goodfellow2014generative} of the form below can be used to regularize the generator
\begin{align}
    \min_{G} \max_{D}\mathbb{E}_{F \!\sim \mathbb{P}^+\!(\!F)}\!\log [D(F)]\!\!+\!\mathbb{E}_{F \!\sim \mathbb{P}_g(\!F)}\!\log [1\!\!-\!\!D(\!F)]
    \label{eq:origin_gan}
\end{align}
After convergence, the distribution of $G_\theta(F)$ ({\it i.e.}, $\mathbb{P}_g(F)$) will be equal to $\mathbb{P}^+(F)$. Although the vanilla GAN  can align the generative distribution to the normal feature distribution, the valuable collected anomalies are only used to construct the input distribution $\mathbb{P}(F)$. Since the discriminator has not been informed of any anomaly information, the discriminator could mistakenly recognize anomalies as normal samples, compromising the model's ability of transforming anomaly images into normal ones. To address this issue, we propose to use the anomaly dataset ${\mathcal{X}}^-$ to enhance the discriminator's ability to distinguish between normal and anomalous samples. To this end, we propose an energy-based discriminator $D(\cdot): \mathcal{F} \rightarrow[0, +\infty)$, which is designed to assign low energy to normal features and high energy to generated and anomalous features. The discriminator is trained by minimizing the following loss
\begin{align}
     L_{D}(D, G_{\theta}) \!=& \mathbb{E}_{ F \sim \mathbb{P}^+(F)}D(F) \!+\! \gamma \mathbb{E}_{F \sim \mathbb{P}_g(F)}[0, a \!- \!D(F)]^+ \nonumber \\
    & + (1-\gamma)\mathbb{E}_{F \sim \mathbb{P}^-(F)}[0, a \!-\! D(F)]^+,
    \label{eq:discriminator_image}
\end{align}
where $[u,v]^+ = \max(u, v)$, $a > 0$ is a threshold, and $\gamma\in (0, 1]$ is the balance coefficient. With the trained discriminator, the generator is encouraged to output features to get lower energy by minimizing 
\begin{align}
    L_{G}(D, G_{\theta})\!  = & \gamma \mathbb{E}_{ F \sim \mathbb{P}_g(F)}D(F) + (1\!-\!\gamma)\mathbb{E}_{F \sim \mathbb{P}^-(F)}D(F) \nonumber \\
    &- \mathbb{E}_{ F \sim \mathbb{P}^+(F)}D(F).
    \label{eq:generator_image}
\end{align}
When \eqref{eq:discriminator_image} and \eqref{eq:generator_image} are used to update the discriminator and generator, it can be proved that the output feature distribution $\mathbb{P}_g(F)$ can align with the normal feature distribution $\mathbb{P}^+(F)$. We give the following theorem.
\begin{theorem}
Let $\mathbb{P}^+(F)$, $\mathbb{P}^-(F)$ and $\mathbb{P}_g(F)$ be the distributions of normal, anomalous, and generated feature maps. Assuming the $\mathbb{P}^+(F)$ and $\mathbb{P}^-(F)$ are two disjoint distributions, and $\gamma \in (0, 1]$.
When Discriminator $D(\cdot): \mathcal{F}\rightarrow [0, +\infty)$ and generator $G(\cdot)$ are updated according to \eqref{eq:discriminator_image} and \eqref{eq:generator_image}, 
there exists a Nash equilibrium of the system $(D^*, G^*)$ such that $\mathbb{P}_g(F) = \mathbb{P}^+(F)$.
\label{theorem:image}
\end{theorem}
\begin{proof}
\renewcommand{\qedsymbol}{}
    Please see the proof in the Supplementary Material.
\end{proof}

The discriminator trained with loss \eqref{eq:discriminator_image} has the ability to identify anomalous features from normal ones. If an anomalous image is generated, it will be recognized by the discriminator, which will then incentivize the auto-encoder to transform it into normal one. On the other hand, when input to the encoder is a normal sample, output from the generator will be similar to the input,
thus the adversarial learning will almost not affect this sample's reconstruction. In this way, we achieve the two appealing properties P1 and P2.

\paragraph{Patch-Level Alignment}
The proposed knowledge-aware discriminator in \eqref{eq:discriminator_image} distinguishes between normal and anomalous features at image-level. In practice, anomalies often only occupy small areas of an image. Hence, to achieve more accurate detection and localization, it is better to perform the alignment at the patch level. If we directly divide an image into numerous patches and then pass them through the pre-trained ResNet and the auto-encoder, it will be very expensive. Fortunately, the feature at position $(h, w)$ ({\it i.e.}, $F(h, w)$) corresponds to a patch in the origin image, thus it can be viewed as the feature of a patch. Similarly, the generated feature $\tilde F(h, w)$ at position $(h, w)$ can be viewed as the generated path feature. As shown in Figure \ref{fig:model}, $a^\ell_+$ and $ a^\ell_-$ represent a patch feature of a normal and abnormal image, respectively.

Denote the distribution of features of normal patches as ${\mathbb{P}}^+(p)$. Obviously, features of all patches from normal images follow the distribution ${\mathbb{P}}^+(p)$. For the generated feature $\tilde{F}(h, w)$ at position $(h, w)$, we denote its distribution as $\mathbb{P}_g^{(h, w)}(p)$, and further denote the generator that reconstructs the $(h, w)$-th patch feature as $G_\theta^{(h, w)}(\cdot)$. To align the distribution of generated patch features with that of normal patch features, unlike the alignment at the image-level, no path-level anomaly annotations are available. For an anomalous image, it contains both normal and anomalous patches. Thus, for images from ${\mathcal{X}}^-$, we only know that its patch features follow a mixture distribution
\begin{align}
    \mathbb{P}_m(p) = \beta \mathbb{P}^+(p)+ (1 - \beta)\mathbb{P}^-(p),
\end{align}
where $\beta \in (0, 1)$ is the unknown mixed ratio of $\mathbb{P}^+(p)$ and $\mathbb{P}^-(p)$, with $\mathbb{P}^-(p)$ denoting the distribution of anomalous {\emph{patches}}. Although the available anomalous knowledge is noisy, the following theorem shows that we can still make use of it to enhance the auto-encoder's ability to recover normal patches from anomalous ones.

\begin{theorem}
Suppose the $\mathbb{P}^+(p)$ and $\mathbb{P}^-(p)$ are two disjoint distributions and $\gamma \in (0, 1]$.
When Discriminator  $D^{(h, w)}(\cdot): \mathcal{P}\rightarrow [0, +\infty)$  and generator $G^{(h, w)}(\cdot)$ are updated by minimizing the two losses
\begin{align}
     L_{D}(D^{(h, w)}, G^{(h, w)}_{\theta}) \!=
     & \gamma \mathbb{E}_{p \sim \mathbb{P}^{(h, w)}_g(p)}[0, \! a \!-\! D^{(h,w)}(p)]^+ \nonumber \\
    & + \!(1\!-\!\gamma)\mathbb{E}_{p \sim \mathbb{P}_m\!(p)}[0,\! a \!- \!D^{(h,\!w)}\!(p)]^+ \nonumber \\
    & + \mathbb{E}_{ p \sim \mathbb{P}^+(p)}D^{(h,w)}(p), 
    \label{eq:discriminator_patch}
\end{align}
\begin{align}
     L_{G}(D^{(h,w)}, G^{(h,w)}_{\theta}) & = \gamma \mathbb{E}_{p \sim \mathbb{P}^{(h,w)}_g(p)}D^{(h,w)}(p) \nonumber \\
    & + (1-\gamma)\mathbb{E}_{p \sim \mathbb{P}_m(p)}D^{(h,w)}(p) \nonumber \\
    & -  \mathbb{E}_{ p \sim \mathbb{P}^+(p)}D^{(h,w)}(p), 
    \label{eq:generator_patch}
\end{align}
there exists a Nash equilibrium of the system $(D^{*(h,w)}, G^{*(h,w)})$ such that $\mathbb{P}^{(h, w)}_g(p) = \mathbb{P}^+(p)$.
\label{theorem:patch}
\end{theorem}
\begin{proof}
\renewcommand{\qedsymbol}{}
    Please see the proof in the Supplementary Material.
\end{proof}

According to \eqref{eq:discriminator_patch}, the discriminator is designed to assign low energy to normal patch features, while assigning high energy to patch features that are generated or from anomalous images. 
Although the patch discriminator seems contradictory in the sense of encouraging to output low and high energy to normal patch features simultaneously, due to the term $\mathbb{E}_{ p \sim \mathbb{P}^+(p)}D(p)$ in \eqref{eq:discriminator_patch} has greater weight, its influence will override that of anomalous images $\mathbb{E}_{p \sim \mathbb{P}_m(p)}[0, a - D^{(h,w)}(p)]^+$, thus leading to the overall low energy for normal patches. With the discriminator capable of outputting low and high energy for normal, anomalous and generated patches, by driving the generator to generate patches with low energy, we can align the distribution of generated patches with normal ones'. Thus, by comparing the input patch features with the reconstructed ones, anomalous areas can be localized accurately.
In experiments, a common patch feature discriminator $D(p)$ is used for all patch feature generators $G^{(h, w)}(p)$.

\subsection{Training and Testing}
\paragraph{Training} For patch feature from layer $s \in \mathcal{S}$, the discriminator $D^s(p)$ can be used to distinguish patch features of all positions as \eqref{eq:discriminator_patch}. Combining discriminators of all layers that participate in the reconstruction, we can get a total discriminator loss and corresponding generator loss $\fontsize{8}{8}\selectfont L_{D} = \sum_{s \in \mathcal{S}} \frac{1}{H^s W^s} \sum_{h=1}^{H^s}\sum_{w=1}^{W^s}  L_{D}(D^s, G^{(h,w)}_{\theta})$ and $L_{G} = \sum_{s \in \mathcal{S}} \frac{1}{H^s W^s} \sum_{h=1}^{H^s}\sum_{w=1}^{W^s} L_{G} (D^s, G^{(h, w)}_{\theta})$,
where $H^s$ and $W^s$ are the height and width of the feature map $F^s$. As the structure shown in Figure \ref{fig:model}, we use the knowledge-aware adversarial learning to regularize the output features of auto-encoder, and train the model by minimizing the following two losses alternatively
\begin{align}
    \min_D L_D, \quad  \min_{\theta} L^+_{rec} + \lambda  L_{G}.
    \label{eq:ed_loss}
\end{align}


\paragraph{Testing} For any test sample $x$, we first obtain its pre-trained feature maps $F^s$, and reconstructed feature maps $\tilde{F}^s$. For pixel-level anomaly detection, we compute the value
\begin{align}
    S^s_{map}(h,w) = 1 - cos\left(F^s(h,w), \tilde{F}^s(h,w)\right)
\end{align}
of position $(h,w)$ in anomaly maps for $s\in {\mathcal{S}}$, and upsample them to the same size as the input image by bi-linear interpolation. Finally, we sum the upsampled error maps of all layers together and smooth the summed map with a Gaussian kernel, yielding the final anomaly score map $S_{map}$.
For image-level anomaly detection, we obtain the average result of the top-k values in the anomaly maps to produce the final anomaly score $S_{img}$.

\begin{table*}[t]
\centering
\setlength{\tabcolsep}{1.5mm}
{\small {\begin{tabular}{c|c|c|ccc|ccc|ccc|ccc}
\toprule
& \multirow{2}{*}{$r_l$} & \multirow{2}{*}{Methods} & \multicolumn{3}{c|}{ISIC2018} & \multicolumn{3}{c|}{Chest X-ray} & \multicolumn{3}{c|}{Br35H} & \multicolumn{3}{c}{OCT} \\
\cmidrule(r){4-6}  \cmidrule(r){7-9} \cmidrule(r){10-12} \cmidrule(r){13-15} 
& & & AUC & F1 & ACC &  AUC & F1 & ACC & AUC & F1 & ACC &  AUC & F1 & ACC \\
\midrule
\multirow{6}{*}{Unsupervised} & \multirow{6}{*}{0\%} & 
GANomaly & 72.54 & 60.95 & 56.99 & 76.14 & 81.18 & 74.36 & 84.43 & 78.89 & 75.83 & 90.52 & 91.09 & 86.16 \\
& & DifferNet & 76.52 & 64.58 & 64.25 & 86.41 & 84.62 & 79.17 & 91.66 & 84.98 & 85.33 & 94.27 & 92.85 & 89.05 \\
& & Fastflow & 80.81 & 69.95 & 70.98 & 90.25 & 86.02 & 80.77 & 93.85 & 87.67 & 86.83 & 94.08 & 92.60 & 88.43\\
& & AE-FLOW & 87.79 & 80.56 & 84.97 & \textbf{92.00} & \textbf{88.92} & \textbf{85.58} & N/A & N/A & N/A & 98.15 & 96.36 & 94.42 \\
& & ReContrast & \textbf{90.15} & \textbf{81.12} & \textbf{86.01} & 84.20 & 83.45 & 77.24 & 96.77 & 93.95 & 93.67 & \textbf{99.60} & \textbf{98.53} & \textbf{97.80}\\
& & Recon & 87.51 & 79.50 & 82.90 & 85.59 & 83.70 & 78.04 & \textbf{97.37} & \textbf{96.23} & \textbf{96.16}  & 99.33 & 97.84 & 96.80 \\
\midrule
\multirow{15}{*}{\shortstack{Weakly \\{Supervised}}} & \multirow{6}{*}{1\%} & PReNet & 73.84 & 67.02 & 69.26 & 86.66  & 84.83 &  80.45 & 81.86 & 76.66 & 75.33 & 60.26 & 85.88 & 75.83\\
& & Deep SAD & 89.52 & 77.13 & 80.83 & 86.82 & 83.40 & 79.33 & 86.83 & 81.41 & 80.67 & 93.04 & 90.93 & 86.50 \\
& & AA-BiGAN & 87.83 & 78.96 & 83.25 & 90.63 & 87.52 & 84.13 & 88.87 & 84.04 & 83.67 & 96.02 & 94.41 & 91.70 \\
& & DRA &90.10 & 78.51 & 82.04 & 92.35 & 88.36 & 85.59 & 90.41 & 84.18 & 84.33 & 91.52 & 91.34 & 86.83 \\
& & DevNet & 88.98 & 77.80 & 80.83 & 92.53 & \textbf{89.41} & 85.57 & 93.94 & 88.29 & 88.33 & 99.38 & 98.75 & 98.13 \\
& & CKAAD (Ours) & \textbf{91.02} & \textbf{82.04} & \textbf{85.84} & \textbf{93.22} & 88.77 & \textbf{86.22} & \textbf{98.89} & \textbf{97.88} & \textbf{97.83} & \textbf{99.88} & \textbf{99.33} & \textbf{99.00} \\
\cmidrule(r){2-15}
& \multirow{6}{*}{2\%} & PReNet & 74.28 & 66.87 & 68.83 & 89.69 &  87.47 & 84.29 & 82.33  & 75.00 & 76.33 & 65.25 & 85.98 & 75.90 \\
& & Deep SAD & 89.86 & 79.18 & 83.94 & 90.16 & 86.56 & 83.97 & 85.91 & 79.48 & 76.33 & 95.26 & 93.18 & 90.00 \\
& & AA-BiGAN & 88.20 & 78.10 & 83.25 & 92.13 & 89.86 & 87.02 & 84.46 & 79.87 & 79.17 & 97.15 & 95.44 & 93.20\\
& & DRA & 90.65 & 79.70 & 82.80 & 92.59 & 89.43 & 86.69 & 95.00 & 89.92 & 89.50 & 94.98 & 92.94 & 89.50 \\
& & DevNet & 89.36 & 78.80 & 83.25 & 94.27 & 90.48 & 87.66 & 94.89 & 88.05 & 88.33 & 99.31 & 98.56 & 97.83 \\
& & CKAAD (Ours) & \textbf{91.68} & \textbf{82.86} & \textbf{86.87} & \textbf{95.23} & \textbf{91.17} & \textbf{88.94} & \textbf{99.07} & \textbf{98.03} & \textbf{98.00} & \textbf{99.91} & \textbf{99.42} & \textbf{99.13} \\
\cmidrule(r){2-15}
& \multirow{6}{*}{5\%} & PReNet & 74.29 & 67.36 & 71.24 & 91.01 & 87.82 & 83.81 & 85.12 & 80.48 & 78.33 & 67.67 & 86.27 & 76.47\\
& & Deep SAD & 89.81 & 78.29 & 82.38 & 92.87 & 88.49 & 85.58 & 91.54 & 83.50 & 83.33 & 95.74 & 93.88 & 91.03\\
& & AA-BiGAN & 88.91 & 78.95 & 82.30 & 92.80 & 89.95 & 86.86 & 95.96 & 91.87 & 91.83 & 97.69 & 96.04 & 94.10 \\
& & DRA & 91.05 & 79.80 & 83.07 & 93.61 & 89.59 & 86.86 & 95.48 & 90.03& 89.67 & 91.91 & 91.39 & 86.57 \\
& & DevNet & 89.93 & 78.63 & 82.64 & 95.77 & 91.23 & 88.94 & 99.52 & 97.80& 97.83 & 99.60 & 98.93 & 98.40 \\
& & CKAAD (Ours) & \textbf{92.66} & \textbf{84.00} & \textbf{87.65} & \textbf{96.57} & \textbf{93.87} & \textbf{92.47} & \textbf{99.79} & \textbf{98.20} & \textbf{98.17} & \textbf{99.93} & \textbf{99.55} & \textbf{99.33} \\
\bottomrule
\end{tabular}}}
\caption{Performance (Average AUC(\%), best F1(\%)) and ACC(\%)) on medical datasets. With $k=1$, we increase the ratio of labeled anomalies $r_l$ in the training set from $0\%$ to $5\%$.}
\label{tab:medical}
\end{table*}

\section{Experiments}
\subsection{Experimental Setups}
\paragraph{Datasets} We mainly conduct experiments on real-world anomaly detection medical datasets and industrial datasets.
\textbf{Medical datasets:} 
{\it i)} \textbf{ISIC2018} ~\cite{tschandl2018ham10000,codella2019skin}: The ISIC2018 challenge dataset (task 3) contains 7 categories and we classify NV (nevus) as normal, the rest 6 categories as abnormal. 
{\it ii)} \textbf{Chest X-rays}~\cite{kermany2018identifying} contains normal and abnormal Chest X-rays scans.
{\it iii)} \textbf{Br35H}\cite{br35h,zhou2024anomalyclip}: Brain Tumor Detection dataset contains non-tumorous and various tumorous images.
{\it iv)} \textbf{OCT}~\cite{kermany2018identifying}: Retinal optical coherence tomography (OCT) contains normal OCT scans and three types of scans with diseases. 
\textbf{Industrial dataset:} {\it i})\textbf{MVTec AD}~\cite{bergmann2019mvtec} is a widely known industrial dataset comprising 15 classes with 5 textures and 10 objects. 
{\it ii})\textbf{Visa}~\cite{zou2022visa} is an industrial dataset containing 12 classes.

\paragraph{Implementation} All images are resized to 256. $\alpha, \gamma$ are set to $0.5$. $\lambda$ is set to $0.02$. Adam optimizer is utilized with $\beta=(0.5,0.999)$. The learning rate for the auto-encoder is set to 1e-03 for medical datasets, 5e-03 for industrial datasets, and 1e-04 for the discriminator. Resnet18 is chosen as the backbone and $\mathcal{S} = \{2, 3\}$ for medical datasets. WideResnet50 and $\mathcal{S} = \{1,2,3\}$ are set for industrial datasets because the anomalies are more subtle. 

\paragraph{Metrics} We report the AUC (Area Under receiver operating characteristic Curve), best F1, Accuracy (ACC) for medical datasets, and image-level AUC, pixel-level AUC and PRO (Per-Region Overlap) for industrial datasets. 

\paragraph{Experimental Scenarios} For medical datasets, we evaluate the performance with consideration of two key parameters for the incomplete anomalous information: 1) the number of types of collected anomalies $k$; 2) the ratio of collected anomalies $r_l$ in the training data. Specifically, we first fix $k=1$ but increase $r_l$ gradually. In this case, anomalies from any anomaly category can be used as the only labeled anomaly class during training, and performance averaged over all anomaly types is reported. Then, we fix $r_l$ but increase $k$ gradually. The $k$ anomaly types are randomly selected from all possible anomaly categories, so the performance is averaged over 5 random experiments. 
For industrial datasets, since there are no anomalies for training, elastic transformation (see details in Supplementary) is used to distort normal images to produce corresponding anomalies and the average performance of all classes is reported.

\subsection{Detection Results on Medical Datasets }

We compared our proposed methods with several unsupervised methods 
GANomaly\cite{akcay2019ganomaly}, 
DifferNet\cite{rudolph2021differnet}, Fastflow\cite{yu2021fastflow}, AE-FLOW\cite{zhao2023aeflow}, ReContrast\cite{guo2023recontrast}, and weakly-supervised methods Deep SAD\cite{ruff2019deepsad}, DevNet\cite{pang2021devnet}, AA-BiGAN\cite{tian2022aa-bigan}, DRA\cite{ding2022dra}, PReNet~\cite{pang2023deep}.
All methods are run under the same settings.

\paragraph{Performance under different ratios of labeled anomalies}
As shown in Table~\ref{tab:medical}, with $k=1$, the performance of unsupervised and weakly supervised methods is reported. Compared to the reconstruction baseline (Recon), our proposed method can enhance the reconstruction model's performance by utilizing labeled anomalies. Using only 1\% anomalies, our method improves the AUC by 3.51\% for ISIC2018, 7.63\% for Chest X-ray, 1.52\% for Br35H.
While adding $r_l$ to 2\% and 5\%, the performance continuously increases. On the other hand, compared to other weakly-supervised methods, our method exhibits the best overall performance of three metrics on all medical datasets.
  
\begin{figure}[tbp]
\centering 
\begin{minipage}{.48\textwidth} 
\centering 
{\includegraphics[width=0.98\linewidth]{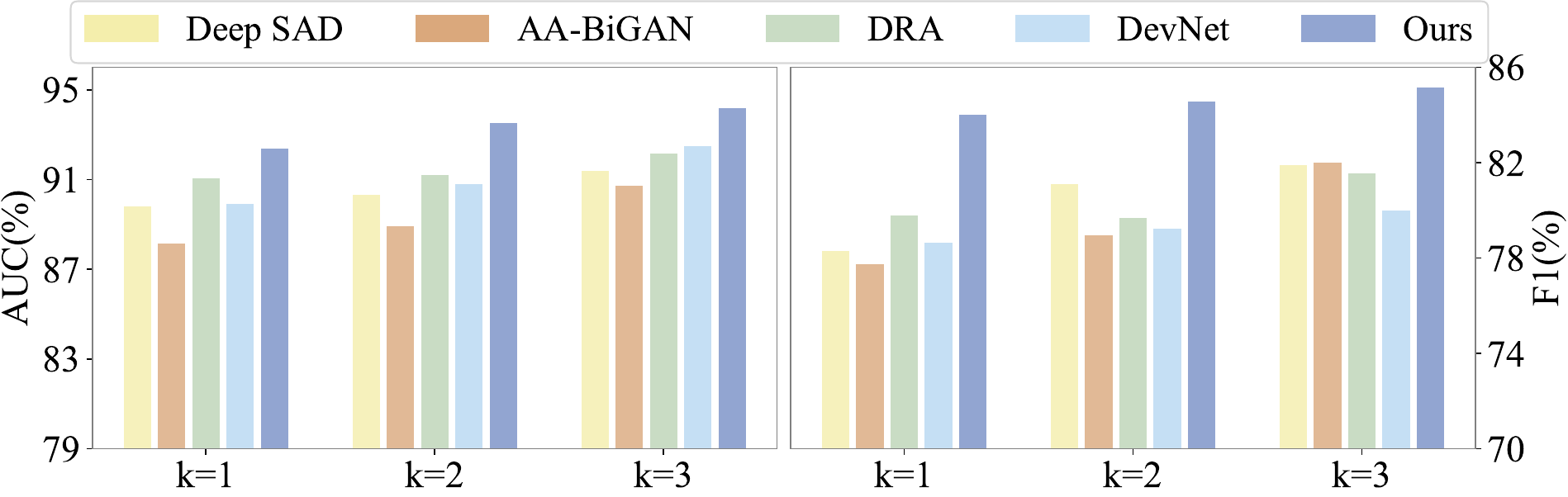}} 
\captionof{figure}{Histograms of performance (Average AUC and F1 of 5 random experiments) with different number ($k$) of labeled anomaly classes on ISIC2018 under $r_l=5\%$.} 
\label{fig:isic_k} 
\end{minipage} 
\end{figure}


\begin{table}[tbp]
    \centering
     \setlength{\tabcolsep}{1mm}
    {\small{
    \begin{tabular}{ccccccc}
    \toprule
       Dataset  &  \multicolumn{3}{c}{MVTec} & \multicolumn{3}{c}{Visa} \\
       \cmidrule(r){1-1} \cmidrule(r){2-4} \cmidrule(r){5-7} 
       Metric & I-AUC & P-AUC & PRO & I-AUC & P-AUC & PRO \\
       \midrule
       PaDiM & 95.4 & 97.5 & 92.1 & 89.1 & 98.1 & 85.9 \\
       DRAEM & 98.0 & 97.3 & 93.0 & 88.7 & 93.5 & 73.1 \\
       RD4AD & 98.5 & 97.8 & 93.9 & 96.0 & 90.1 & 70.9 \\
       PatchCore & 99.1 & 98.1 & 93.4 & 95.1 & \textbf{98.8} & 91.2 \\
       DesTseg & 98.6 & 97.9 & 92.6 & 92.3 & 98.4 & 92.3 \\ 
       RD++ & 99.4 & 98.3 & 95.0 & 95.9 & 98.7 & 93.4\\
       DMAD & \textbf{99.5} & 98.2 & 90.6 & 95.5 & 98.6 & 91.3 \\
       D3AD & 97.2 & 97.4 & 93.3 & 96.0 & 97.9 & \textbf{94.1} \\
       CKAAD (Ours) & \textbf{99.5} & \textbf{98.4} & \textbf{95.2} & \textbf{96.7} & \textbf{98.8} & 94.0 \\
    \bottomrule
    \end{tabular}}}
     \caption{Performance on industrial datasets. Average Image-level / Pixel-level AUC (I-AUC / P-AUC) and PRO of all classes are reported.}
    \label{tab:industrial}
\end{table}

\paragraph{Performance under different types of labeled anomalies}
As Figure \ref{fig:isic_k} shows, we evaluate the performance with different numbers of labeled anomaly types on ISIC2018 which has 6 kinds of anomalies. With only 5\% anomalies, the type $k$ is increased from 1 to 3, and the histogram exhibits the average AUC and best F1 of 5 weakly-supervised methods. With the increasing number of observed anomaly types, weakly-supervised methods can continuously improve detection performance. Among them, our method gets the best overall performance under all $k$ settings.

\subsection{Detection Results on Industrial Datasets}
We also evaluate fine-grained anomaly detection on industrial datasets and 
report the performance in Table \ref{tab:industrial}.
Compared to recent state-of-the-art approaches 
PaDim~\cite{defard2021padim},
DRAEM~\cite{zavrtanik2021draem},
RD4AD~\cite{deng2022RD}, 
PatchCore~\cite{roth2022patchcore},
DesTseg~\cite{zhang2023destseg}, 
RD++~\cite{tien2023revisiting},
DMAD~\cite{liu2023diversity_DMAD} and
D3AD~\cite{Tebbe_2024_CVPR_D3AD},
our method still gets competitive performance in both image-level and pixel-level detection, demonstrating our effectiveness in fine-grained anomaly detection.

\subsection{Ablation Study}

There are several strategies of utilizing anomalies, including the Recon \eqref{eq:gcd}, ReconSub \eqref{eq:L_rec}, pure GAN \eqref{eq:origin_gan}, coarse-knowledge-aware image-level adversarial learning (CKAImg)\eqref{eq:discriminator_image}\eqref{eq:generator_image} and patch-level adversarial learning (CKAPatch)\eqref{eq:discriminator_patch}\eqref{eq:generator_patch}. To compare these strategies, we conduct experiments with $r_l=5\%, k=1$ for medical datasets and distorted anomalies for industrial datasets.
As shown in Table \ref{tab:compare_loss}, compared to Recon, ReconSub can decrease performance a lot on MVTec, indicating it is unsuitable for fine-grained anomaly localization. Compared to pure GAN, CKAImg improves the detection performance more, that is because the energy discriminator is knowledge-aware of anomalies, leading the reconstructed feature to be better aligned with normal ones. Furthermore, employing CKAPatch can make the discriminator be better aware of anomalous areas, thus guiding the auto-encoder to output normal patch features more surely, further boosting the detection performance. 


\begin{table}[t]
    \centering
    \setlength{\tabcolsep}{1mm}
    {\small{
    \begin{tabular}{cccccccc}
    \toprule
        \multirow{2}{*}{Strategy} & \multicolumn{2}{c}{ISIC2018} & \multicolumn{2}{c}{Chest X-Ray} & \multicolumn{3}{c}{MVTec} \\
         \cmidrule(lr){2-3} \cmidrule(lr){4-5}  \cmidrule(lr){6-8} 
         & AUC & F1 & AUC& F1 & I-AUC & P-AUC & PRO \\
        \cmidrule(lr){1-1} \cmidrule(lr){2-3} \cmidrule(lr){4-5}  \cmidrule(lr){6-8}
        Recon & 87.5 & 79.5 & 85.6&83.7 & 99.1&98.2&94.4 \\
        ReconSub & 89.2&79.0 & 92.3&87.3 & 91.2&95.0&86.2 \\
        GAN & 91.0&82.0 & 95.2&91.1 & 99.2&98.2&94.9 \\
        CKAImg & 91.4&82.3 & 96.0&92.0 & 99.3&98.3&95.1 \\
        CKAPatch & \textbf{92.6}&\textbf{84.0} & \textbf{96.6}&\textbf{93.9}  & \textbf{99.5}&\textbf{98.4}&\textbf{95.2} \\
       \bottomrule
    \end{tabular}}
    }
 \caption{Ablation of different strategies.}
 \label{tab:compare_loss}
\end{table}

\begin{figure}[tbp]
    \centering
\subfigure[normal]{%
    \includegraphics[width=0.40\linewidth]{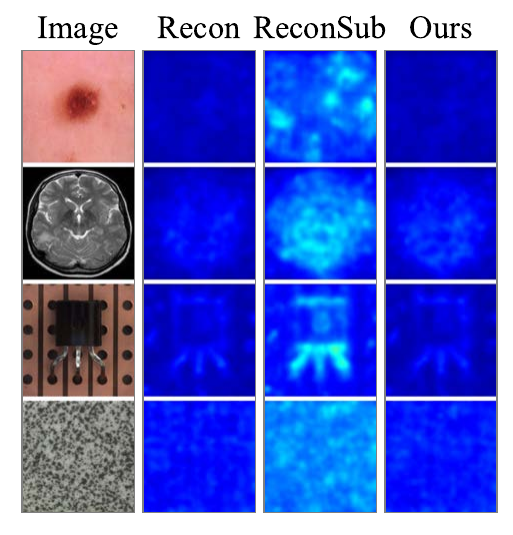}%
}
\subfigure[abnormal]{%
    \includegraphics[width=0.54\linewidth]{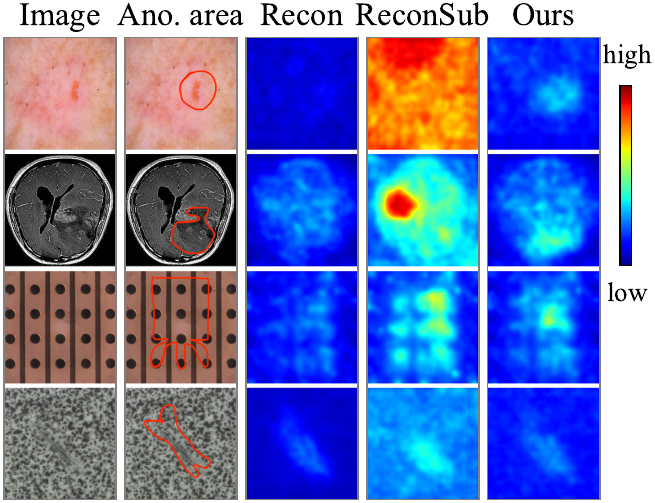}%
}
\caption{Visualization of detection results. Images, anomaly maps of Recon, ReconSub, and our method are shown in the first column and last three columns respectively. The second column in (b) circles anomalous areas by red lines.}
\label{fig:ablation}
\end{figure}

Figure \ref{fig:ablation} exhibits image error maps of different strategies. The reconstruction quality can be reflected by the error maps in which high values represent the area not reconstructed well. The visualization result shows that only reconstructing normal samples can also generalize to anomalies, thus anomalous areas can not be detected in the error maps. ReconSub can negatively influence the reconstruction quality of normal samples while not accurately localizing the anomalous areas. The proposed patch-level knowledge-aware method can reconstruct normal images well while having the ability to localize the anomalous areas.

\section{Conclusion}

In this paper, we propose to leverage a small coarse-labeled anomaly dataset to boost the anomaly detection performance of the feature reconstruction model. 
To achieve our goal, we first introduce the energy-based discriminators which can be aware of anomalous information to guide the feature alignment at image-level, then transfer the alignment to patch-level by only using the coarse image labels to better localize the anomalous area.
Experimental results on medical and industrial datasets prove the effectiveness of our method.

\section*{Acknowledgments}
This work is supported by the National Natural Science Foundation of China (No. 62276280, 62276279), Guangzhou Science and Technology Planning Project (No. 2024A04J9967), Guangdong Basic and Applied Basic Research Foundation (2024B1515020032).


\bibliography{aaai25}

\appendix
\setcounter{theorem}{0}


\twocolumn[
\begin{@twocolumnfalse}
      \centering%
      {\LARGE\bf Supplementary Material of Boosting Fine-Grained Visual Anomaly Detection with \\ Coarse-Knowledge-Aware Adversarial Learning \par}%
      \vspace{15mm}
\end{@twocolumnfalse}
]

\section{Theoretical Analysis}
\label{appedix:theoretical_analysis}

Before the proof of theorems, we first introduce Lemma \ref{lemma:appendix_lemma1}.
\newtheorem{lemma}{Lemma}

\begin{lemma}
    If $\mathbb{P}_1$ and $\mathbb{P}_2$ are probability densities, $\|\mathbb{P}_1, \mathbb{P}_2\|_{TV}$ is the total variation distance of the two distributions, $f(x): \mathcal{X} \rightarrow [-1, 1]$, then we have
    \begin{align}
         \|\mathbb{P}_1, \mathbb{P}_2\|_{TV} = -\inf_{-1 \leq f(x) \leq 1}\mathbb{E}_{x \sim \mathbb{P}_1}[f(x)] - \mathbb{E}_{x\sim \mathbb{P}_2}[f(x)].
    \end{align}
    And there exist a $f^*$ such that $\|\mathbb{P}_1, \mathbb{P}_2\|_{TV} = -\left\{\mathbb{E}_{x \sim \mathbb{P}_1}[f^*(x)] - \mathbb{E}_{x\sim \mathbb{P}_2}[f^*(x)]\right\}$.
    \label{lemma:appendix_lemma1}
\end{lemma}

\begin{proof}
    \renewcommand{\qedsymbol}{}
    If $f$ is bounded between -1 and 1, take $\mu = \mathbb{P}_1 - \mathbb{P}_2$, which is a signed measure, we get
    \begin{align}
        & |\mathbb{E}_{x \sim \mathbb{P}_1}[f(x)] -\mathbb{E}_{x \sim \mathbb{P}_2}[f(x)]| \nonumber \\
        = &|\int f d\mathbb{P}_1 - \int f d\mathbb{P}_2| \nonumber\\
        = &|\int f d\mu| \nonumber \\
        \leq &\int |f|d|\mu| \leq \int 1 d|\mu| \nonumber \\
        = & |\mu| = \|\mathbb{P}_1 - \mathbb{P}_2\|_{TV}.
    \end{align}
    Since $\|\cdot\|_{TV} \geq 0$, we can conclude $\mathbb{E}_{x \sim \mathbb{P}_1}[f(x)] - \mathbb{E}_{p\sim \mathbb{P}_2}[f(x)] \geq - \|\mathbb{P}_1, \mathbb{P}_2\|_{TV}$.
    
   Assuming $(P, Q)$ is Hahn decomposition of $\mu$. Then , we can define $f^* = \mathds{1}_{Q} - \mathds{1}_{P}$. By construction, then
    \begin{align}
        & \mathbb{E}_{p \sim \mathbb{P}_1}[f^*(x)] -\mathbb{E}_{x \sim \mathbb{P}_2}[f^*(x)] \nonumber \\
        = & \int f^* d\mu = \mu(Q) - \mu(P) \nonumber \\
        = & -(\mu(P) - \mu(Q)) = -\|\mu\|_{TV} = -\|\mathbb{P}_1, \mathbb{P}_2\|_{TV}
    \end{align}
    Thus, there exists $f^*$ such that $\|\mathbb{P}_1, \mathbb{P}_2\|_{TV} = -\left\{\mathbb{E}_{x \sim \mathbb{P}_1}[f^*(x)] - \mathbb{E}_{x\sim \mathbb{P}_2}[f^*(x)]\right\}$.
\end{proof}

\begin{theorem}
Let $\mathbb{P}^+(F)$, $\mathbb{P}^-(F)$ and $\mathbb{P}_g(F)$ be the distributions of normal, anomalous, and generated feature maps. Assuming the $\mathbb{P}^+(F)$ and $\mathbb{P}^-(F)$ are two disjoint distributions, and $\gamma \in (0, 1]$,
when Discriminator $D(\cdot): \mathcal{F}\rightarrow [0, +\infty)$ and generator $G(\cdot)$ are updated according to 
\begin{align}
      L_{D}(D, G_{\theta}) &= \mathbb{E}_{ F \sim \mathbb{P}^+\!(F)}D(F) \!+ \!\gamma \mathbb{E}_{F \sim \mathbb{P}_g\!(F)}[0, a\! -\! D(F)]^+ \nonumber \\
        & + (1-\gamma)\mathbb{E}_{F \sim \mathbb{P}^-(F)}[0, a - D(F)]^+,
    \label{eq:appendix_discrminator_image_theorem}
\end{align}
\begin{align}
      L_{G}(D, G_{\theta}) & = \gamma \mathbb{E}_{ F \sim \mathbb{P}_g(F)}D(F) + (1-\gamma)\mathbb{E}_{F \sim \mathbb{P}^-(F)}D(F) \nonumber \\
    &- \mathbb{E}_{ F \sim \mathbb{P}^+(F)}D(F),
    \label{eq:appendix_generator_image_theorem}
\end{align}
there exists a Nash equilibrium of the system $(D^*, G^*)$ such that $\mathbb{P}_g = \mathbb{P}^+$.
\end{theorem}
\begin{proof}
\renewcommand{\qedsymbol}{}
\label{proof:theorem1}
    First, we prove that there exists an optimal discriminator. $D: \mathcal{F}\rightarrow [0, +\infty)$ is a measurable function composed by a neural network, mapping the feature map $F$ into a real value. Let $D'(p) = min(D(F), a)$, then, we has $L_D(D', G_{\theta}) \leq L_D(D, G_\theta)$. Therefore, a function $D^*: \mathcal{P}\rightarrow[0, +\infty)$ is optimal if and only if $(D^*)'$ is. 
    Furthermore, it is optimal if and only if $L_D(D^*, G_{\theta}) \leq L_D(D, G_\theta)$ for all $D: \mathcal{F} \rightarrow [0, a]$. 
    So we consider if there is an optimal discriminator for the problem $\min_{0 \leq D(F) \leq a} {L_{D}(D, G_{\theta})}$.

    If $0 \leq D(F) \leq a$, we have
    \begin{align}
        L_{D}(D, G_{\theta}) &= \mathbb{E}_{ F \sim \mathbb{P}^+\!(F)}D(F) \!+\! \gamma \mathbb{E}_{F \sim \mathbb{P}_g\!(F)}[0, a \!-\! D(F)]^+ \nonumber \\
        & + (1-\gamma)\mathbb{E}_{F \sim \mathbb{P}^-(F)}[0, a - D(F)]^+ \nonumber \\
        = & a + \mathbb{E}_{F \sim \mathbb{P}^+(F)} - \gamma \mathbb{E}_{F \sim \mathbb{P}_g(F)}D(F) \nonumber \\
        & - (1-\gamma) \mathbb{E}_{F \sim \mathbb{P}^-(F)}D(F)
    \end{align}
Minimizing $L_D(D, G_{\theta})$ is equivalent to finding for the lower bound of $L_D(D, G_{\theta})$, we define $\mathbb{P}'_m(F) = \gamma\mathbb{P}_g(F) + (1-\gamma)\mathbb{P}^-(F), D(F) = \frac{a}{2}f(F) + \frac{a}{2}$, thus
\begin{align}
    & \inf_{0 \leq D(F) \leq a}L_D(D, G_{\theta}) \nonumber \\
    = & a + \inf_{0 \leq D(F) \leq a}\{\mathbb{E}_{F \sim \mathbb{P}^+(F)}D(F) - \gamma \mathbb{E}_{F \sim \mathbb{P}_g(F)}D(F) \nonumber \\
    & - (1-\gamma) \mathbb{E}_{F \sim \mathbb{P}^-(F)}D(F)\} \nonumber \\
    = & a + \inf_{0\leq D(F) \leq a}\{\mathbb{E}_{F \sim \mathbb{P}^+(F)}D(F) - \mathbb{E}_{F \sim \mathbb{P}'_m(F)}D(F)\} \nonumber \\
    = & a + \inf_{-1 \leq f(F) \leq 1}\{\mathbb{E}_{F \sim \mathbb{P}^+(F)}f(F) - \mathbb{E}_{F \sim \mathbb{P}'_m(F)}f(F)\}.
\end{align}
According to Lemma \ref{lemma:appendix_lemma1}, there exists an optimal $f^*(F): \mathcal{F} \rightarrow [-1, 1]$ such that $\inf_{-1 \leq f(F) \leq 1}\{\mathbb{E}_{F \sim \mathbb{P}^+(F)}f(F) - \mathbb{E}_{F \sim \mathbb{P}'_m(F)}f(F)\} = -\|\mathbb{P}^+, \mathbb{P}'_m\|_{TV}$. The optimal discriminator is $D^*(F) = \frac{a}{2}f^*(F) + \frac{a}{2}$. Thus we have
\begin{align}
    L_D(D^*, G_{\theta}) = a - \frac{a}{2}\|\mathbb{P}^+, \mathbb{P}'_m\|_{TV}.
\end{align}
Minimizing \eqref{eq:appendix_generator_image_theorem} is to calculate to total variation distance between the distribution $\mathbb{P}^+$ and mixed distribution $\mathbb{P}'_m$. Bring the optimal $D^*$ into $L_G(D, G_{\theta})$, we can get
\begin{align}
    L_G(D^*, G_{\theta}) = -L_D(D^*, G_{\theta}) + a = \frac{a}{2}\|\mathbb{P}^+, \mathbb{P}'_m\|_{TV}.
\end{align}
Assuming $\mathbb{P}^+(F)$ and $\mathbb{P}^-(F)$ are two disjoint distributions, we divide the entire feature space $\mathcal{F}$ into two subspaces: $\mathcal{P}^+ = Supp(\mathbb{P}^+)$, $\mathcal{P}^- = \mathcal{P} \setminus \mathcal{P}^+$. On the subspace $\mathcal{P}^+$, $\mathbb{P}^- = 0, \int_{\mathcal{P}^+}\mathbb{P}^+ dF = 1 \geq \int_{\mathcal{P}^+}\mathbb{P}_g dF$. On the subspace $\mathcal{P}^-, \int_{\mathcal{P}^-}\mathbb{P}dF = 1, \mathbb{P}^+ = 0$. Thus we have

\begin{align}
    & L_G(D^*, G_{\theta}) \nonumber \\
    = & \frac{a}{2}\|\mathbb{P}^+, \mathbb{P}'_m\|_{TV} = \frac{a}{4}\int_{\mathcal{P}}|\mathbb{P}^+ -\gamma \mathbb{P}_g - (1-\gamma)\mathbb{P}^-| dF \nonumber \\
    = & \frac{a}{4}[\int_{\mathcal{P}^+}|\mathbb{P}^+ -\gamma \mathbb{P}_g| dF + \int_{\mathcal{P}^-}|\gamma \mathbb{P}_g + (1-\gamma)\mathbb{P}^-| dF]. \\
\end{align}
In the subspace $\mathcal{P}^+$, we have 
\begin{align}
    & \int_{\mathcal{P}^+}|\mathbb{P}^+ -\gamma \mathbb{P}_g| dF \nonumber \\
    \geq & |\int_{\mathcal{P}^+}\mathbb{P}^+ -\gamma \mathbb{P}_g dF| \nonumber \\
    = & |\int_{\mathcal{P}^+} \mathbb{P}^+ dF - \gamma \int_{\mathcal{P}^+}\mathbb{P}_g dF| \nonumber \\
    = & \int_{\mathcal{P}^+} \mathbb{P}^+ dF - \gamma \int_{\mathcal{P}^+}\mathbb{P}_g dF
\end{align}
Thus, we have
\begin{align}
    & L_G(D^*, G_{\theta}) \nonumber \\
    \geq & \frac{a}{4}[\int_{\mathcal{P}^+} \gamma(\mathbb{P}^+ - \mathbb{P}_g) + (1-\gamma)\mathbb{P}^+ dF \nonumber \\
    & + \int_{\mathcal{P}^-} (\gamma \mathbb{P}_g + (1-\gamma)\mathbb{P}^-) dF] \nonumber \\
    = & \frac{a}{4}[\gamma \int_{\mathcal{P}^+}(\mathbb{P}^+ - \mathbb{P}_g) dF + (1-\gamma)\int_{\mathcal{P}^+}\mathbb{P}^+dF \nonumber \\
    & + \gamma \int_{\mathcal{P}^-} \mathbb{P}_g dF + (1-\gamma)\int_{\mathcal{P}^-}\mathbb{P}^- dF] \nonumber \\
    \geq & \frac{a}{4}[0 + (1-\gamma) + 0 + (1-\gamma)] = \frac{a(1-\gamma)}{2}. 
\end{align}

Bring $G^*_{\theta}$ such that $\mathbb{P}_g = \mathbb{P}^+$ into $L_G(D^*. G_{\theta})$, we can get
\begin{align}
    L_G(D^*, G^*_{\theta}) = & \frac{a}{4}\int_{\mathcal{P}}|\mathbb{P}^+ - \gamma \mathbb{P}^+ - (1-\gamma)\mathbb{P}^-| dF \nonumber \\
    = & \frac{a}{4}[(1-\gamma)\int_{\mathcal{P}^+} \mathbb{P}^- dF (1-\gamma)\int_{\mathcal{P}^-} \mathbb{P}^+ dF] \nonumber \\
    = & \frac{a}{4}[(1-\gamma) + (1-\gamma)] = \frac{a(1-\gamma)}{2}
\end{align}
The optimal $G^*$ can make loss $L_G$ be equal to its lower bound, so we conclude that there exists a Nash equilibrium system $(D^*, G^*)$ such that $\mathbb{P}_g = \mathbb{P}^+$.

\end{proof}

\begin{theorem}
Let $\mathbb{P}^+(p)$, $\mathbb{P}^-(p)$ and $\mathbb{P}^{(h, w)}_g(p)$ be the distributions of normal, anomalous, and generated patch features, $\mathbb{P}_m(p) = \beta \mathbb{P}^+(p) + (1-\beta)\mathbb{P}^-(p), \beta \in (0, 1)$ is the unknown mixed ratio. Assuming the $\mathbb{P}^+(p)$ and $\mathbb{P}^-(p)$ are two disjoint distributions, and $\gamma \in (0, 1]$,
when Discriminator  $D^{(h, w)}(\cdot): \mathcal{P}\rightarrow [0, +\infty)$  and generator $G^{(h, w)}(\cdot)$ are updated according to
{\small
\begin{align}
       L_{D}(D^{(h, w)}, G^{(h, w)}_{\theta}) = 
     & \gamma \mathbb{E}_{p \sim \mathbb{P}^{(h, w)}_g(p)}[0, a - D^{(h,w)}(p)]^+ \nonumber \\
    & + (1-\gamma)\mathbb{E}_{p \sim \mathbb{P}_m(p)}[0, a - D^{(h,w)}(p)]^+ \nonumber \\
    & + \mathbb{E}_{ p \sim \mathbb{P}^+(p)}D^{(h,w)}(p), 
    \label{eq:appendix_discrminator_theorem_patch}
\end{align}
\begin{align}
     L_{G}(D^{(h,w)}, G^{(h,w)}_{\theta}) & = \gamma \mathbb{E}_{p \sim \mathbb{P}^{(h,w)}_g(p)}D^{(h,w)}(p) \nonumber \\
    & + (1-\gamma)\mathbb{E}_{p \sim \mathbb{P}_m(p)}D^{(h,w)}(p) \nonumber \\
    & -  \mathbb{E}_{ p \sim \mathbb{P}^+(p)}D^{(h,w)}(p), 
    \label{eq:appendix_generator_theorem_patch}
\end{align}
}
there exists a Nash equilibrium of the system $(D^{*(h,w)}, G^{*(h,w)})$ such that $\mathbb{P}^{(h,w)}_g = \mathbb{P}^+$.
\end{theorem}

\begin{proof}
    \label{proof:theorem2}
    \renewcommand{\qedsymbol}{}
    For simplicity, in the following proof, $G^{(h, w)}_{\theta}$, $D^{(h, w)}$, $\mathbb{P}^{(h,w)}_{g}(p)$ are represented as $G_{\theta}$, $D$, $\mathbb{P}_g(p)$, respectively.
    First, we prove that there exists an optimal discriminator. $D: \mathcal{P}\rightarrow [0, +\infty)$ is a measurable function composed by a neural network, mapping the patch feature $p$ into a real value. Let $D'(p) = min(D(p), a)$, then, we has $L_D(D', G_{\theta}) \leq L_D(D, G_\theta)$. Therefore, a function $D^*: \mathcal{P}\rightarrow[0, +\infty)$ is optimal if and only if $(D^*)'$ is. 
    Furthermore, it is optimal if and only if $L_D(D^*, G_{\theta}) \leq L_D(D, G_\theta)$ for all $D: \mathcal{P} \rightarrow [0, a]$. 
    So we consider if there is an optimal discriminator for the problem $\min_{0 \leq D(p) \leq a} {L_{D}(D, G_{\theta})}$.

    If $0 \leq D(p) \leq a$, we have
    \begin{align}
        & L_D(D, G_{\theta}) \nonumber \\
    = & \mathbb{E}_{ p \sim \mathbb{P}^+(p)}D(p) + \gamma \mathbb{E}_{p \sim \mathbb{P}_g(p)}[0, a - D(p)]^+ \nonumber \\
    & + (1-\gamma)\mathbb{E}_{p \sim \mathbb{P}_m(p)}[0, a - D(p)]^+ \nonumber \\
    = & \mathbb{E}_{ p \sim \mathbb{P}^+(p)}D(p) + \gamma \mathbb{E}_{p \sim \mathbb{P}_g(p)}[a - D(p)] \nonumber \\
    & + (1-\gamma)\mathbb{E}_{p \sim \mathbb{P}_m(p)}[a - D(p)] \nonumber \\
    = & a + \mathbb{E}_{ p \sim \mathbb{P}^+(p)}D(p) - \gamma \mathbb{E}_{p \sim \mathbb{P}_g(p)}\!D(p) \nonumber \\
    & - (1-\gamma)\mathbb{E}_{p \sim \mathbb{P}_m(p)}D(p)\nonumber \\
    = & a + \mathbb{E}_{ p \sim \mathbb{P}^+(p)}D(p) - \gamma \mathbb{E}_{p \sim \mathbb{P}_g(p)}D(p)\nonumber \\
    & - (1-\gamma)\beta \mathbb{E}_{p \sim \mathbb{P}^+(p)}D(p) - (1-\gamma)(1-\beta)\mathbb{E}_{p \sim \mathbb{P}^-}D(p)\nonumber \\
    = & a \!+\! (1 \!-\! \beta \!+\! \beta\gamma)[\mathbb{E}_{p \sim \mathbb{P}^+}D(p) \!-\! \frac{\gamma}{(1 \!-\! \beta \!+\! \beta\gamma)} \mathbb{E}_{p \sim \mathbb{P}_g(p)}D(p) \nonumber \\
    & - \frac{(1-\gamma)(1-\beta)}{1 - \beta + \beta\gamma}\mathbb{E}_{p \sim \mathbb{P}^-(p)}D(p)]
    \end{align}
    Minimizing $L_D(D, G_{\theta})$ is equivalent to finding for the lower bound of $L_D(D, G(\theta)$.
    We define $\eta=\frac{\gamma}{1-\beta+\beta\gamma}, \mathbb{P}'_m = \eta \mathbb{P}_g(p) + (1-\eta)\mathbb{P}^-(p)$, thus the lower bound can be represented as 
    \begin{align}
         & \inf_{0\leq D(p) \leq a} L_D(D, G_{\theta}) \nonumber \\
         = & a + \frac{\gamma}{\eta} \inf_{0\leq D(p) \leq a}\{\mathbb{E}_{p \sim \mathbb{P}^+}D(p) - [\eta \mathbb{E}_{p \sim \mathbb{P}_g(p)}D(p) \nonumber \\
         & + (1-\eta) \mathbb{E}_{p \sim \mathbb{P}^-(p)}D(p)]\} \nonumber \\
         = &a + \frac{\gamma}{\eta} \inf_{0\leq D(p) \leq a}\{\mathbb{E}_{p \sim \mathbb{P}^+}D(p) -\mathbb{E}_{p \sim \mathbb{P}'_m}D(p)\} \nonumber \\
         = & a + \frac{\gamma a}{2\eta} \inf_{-1 \leq f(p) \leq 1}\{\mathbb{E}_{p\sim \mathbb{P}^+}f(p) - \mathbb{E}_{p\sim \mathbb{P}'_m}f(p)\}.
    \end{align}

    There is a function $f(p): \mathcal{P}\rightarrow[-1,1]$, and $D(p) = \frac{a}{2}f(p) + \frac{a}{2}$. 
    
    According to Lemma \ref{lemma:appendix_lemma1}, There exists an optimal $f^*(p): \mathcal{P}\rightarrow [-1,1]$ such that $\inf_{-1 \leq f(p) \leq 1} \mathbb{E}_{p\sim \mathbb{P}^+}[f^*(p)] - \mathbb{E}_{p\sim \mathbb{P}'_m}[f^*(p)] = - \|\mathbb{P}^+, \mathbb{P}'_m\|_{TV}$. And the optimal discriminator is $D^*(p) = \frac{a}{2}f^*(p) + \frac{a}{2}$.
    Thus, we have
    \begin{align}
        L_D(D^*, G_{\theta}) = & a - \frac{\gamma a}{2\eta} \|\mathbb{P}^+, \mathbb{P}'_m\|_{TV}
    \end{align}
    Minimizing the equation \eqref{eq:appendix_discrminator_theorem_patch} will get the total variation distance between the distribution $\mathbb{P}^+(p)$ and mixed distribution $\mathbb{P}'_m(p) = \eta \mathbb{P}_g(p) + (1-\eta)\mathbb{P}^-(p)$.
    Bring the optimal $D^*(p)$ into $L_G(D, G_\theta)$, we can get
    \begin{align}
        L_G(D^*, G_{\theta}) =& \gamma \mathbb{E}_{p\sim \mathbb{P}_g} D^*(p) + (1- \gamma) \mathbb{E}_{p \sim \mathbb{P}_m}D^*(p) \nonumber \\
        & - \mathbb{E}_{p \sim \mathbb{P}^+}D^*(p) \nonumber \\
        =& -L_{D}(D^*, G_{\theta}) + a \nonumber \\
        =& \frac{\gamma a}{2\eta}\|\mathbb{P}^+, \mathbb{P}'_m\|_{TV}
    \end{align}
     Assuming $\mathbb{P}^+$ and $\mathbb{P}^-$ are two disjoint distributions, we divide the entire patch feature space $\mathcal{P}$ into two subspaces: $\mathcal{P}^+ = Supp(\mathbb{P}^+)$, $\mathcal{P}^- = \mathcal{P} \setminus \mathcal{P}^+$. On the subspace $\mathcal{P}^+$, $\int_{\mathcal{P}^+}\mathbb{P}^+dp = 1, \mathbb{P}^-=0$. On the subspace $\mathcal{P}^-$, $\int_{\mathcal{P}^-} \mathbb{P}^- dp = 1, \mathbb{P}^+ = 0$.
    Thus we have
    \begin{align}
        & L_G(D^*, G_{\theta}) \nonumber \\
        =& \frac{\gamma a}{2\eta}\|\mathbb{P}^+, \mathbb{P}'_m\|_{TV} \nonumber \\
        = & \frac{\gamma a}{4\eta} \int_{\mathcal{P}} |\mathbb{P}^+ - \eta \mathbb{P}_g - (1-\eta)\mathbb{P}^-| dp \nonumber \\
        = & \frac{\gamma a}{4\eta} \int_{\mathcal{P}^+} |\mathbb{P}^+ - \eta \mathbb{P}_g - (1-\eta)\mathbb{P}^-| dp \nonumber \\
        & + \frac{\gamma a}{4\eta} \int_{\mathcal{P}^-} |\mathbb{P}^+ - \eta \mathbb{P}_g - (1-\eta)\mathbb{P}^-| dp \nonumber \\
        = & \frac{\gamma a}{4\eta} \left\{\int_{\mathcal{P}^+} |\mathbb{P}^+ - \eta \mathbb{P}_g| dp + \int_{\mathcal{P}^-} \eta \mathbb{P}_g + (1-\eta)\mathbb{P}^- dp\right\} \nonumber \\
    \end{align} 
    In subspace $\mathcal{P}^+$, we have that
    \begin{align}
        \int_{\mathcal{P}^+}\mathbb{P}^+ dp = 1 \geq \int_{\mathcal{P}^+}\mathbb{P}_g dp \geq \eta \int_{\mathcal{P}^+}\mathbb{P}_g dp,
    \end{align}
    thus we can get
    \begin{align}
        \int_{\mathcal{P}^+} (\mathbb{P}^+ - \mathbb{P}_g) dp \geq 0, \int_{\mathcal{P}^+} (\mathbb{P}^+ - \eta \mathbb{P}_g) dp \geq 0.
    \end{align}
    So, there is
    \begin{align}
        & L_G(D^*, G_{\theta}) \nonumber \\
        = & \frac{\gamma a}{4\eta} \left\{ \int_{\mathcal{P}^+} |\mathbb{P}^+ - \eta \mathbb{P}_g| dp + \int_{\mathcal{P}^-} \eta \mathbb{P}_g + (1-\eta)\mathbb{P}^- dp \right\}\nonumber \\
        \geq & \frac{\gamma a}{4\eta} \left\{\int_{\mathcal{P}^+} (\mathbb{P}^+ - \eta \mathbb{P}_g) dp + \int_{\mathcal{P}^-} \eta \mathbb{P}_g + (1-\eta)\mathbb{P}^- dp\right\} \nonumber \\
        = & \frac{\gamma a}{4\eta} \left\{\int_{\mathcal{P}^+} \eta (\mathbb{P}^+ - \mathbb{P}_g) dp + \int_{\mathcal{P}^+} (1-\eta)\mathbb{P}^+ dp \right\}\nonumber \\
        & + \frac{\gamma a}{4\eta}  \left\{\int_{\mathcal{P}^-} \eta \mathbb{P}_g + (1-\eta)\mathbb{P}^- dp\right\} \nonumber \\
        \geq & \frac{\gamma a}{4\eta} \left\{0 + (1-\eta) + (1-\eta)\right\} \nonumber \\
        = &\frac{\gamma a (1-\eta)}{2\eta}.
    \end{align}
    
    Bring $G^*_{\theta}$ such that $\mathbb{P}_g = \mathbb{P}^+$ into $L_G(D^*, G_{\theta})$, we can get
    \begin{align}
        L_G(D^*, G^*_{\theta}) = & \frac{\gamma a}{4\eta} \int_{\mathcal{P}} |\mathbb{P}^+ - \eta \mathbb{P}^+ - (1-\eta)\mathbb{P}^-| dp \nonumber \\
        = & \frac{\gamma a}{4\eta} \int_{\mathcal{P}} (1-\eta)[\mathbb{P}^+ + \mathbb{P}^-] dp \nonumber \\
        = & \frac{\gamma a (1-\eta)}{2\eta}.
    \end{align}
    So we conclude there exists a Nash equilibrium system $(D^*, G^*)$ such that $\mathbb{P}_g(p) = \mathbb{P^+}(p)$. That is, there exists a Nash equilibrium system such that $\mathbb{P}^{(h, w)}(p) = \mathbb{P}^+(p)$.
    \label{proof:theorem}
    \end{proof}

\section{Datasets}
\label{appendix:datasets}
\subsection{Medical Datasets}
\begin{itemize}
    \item \textbf{ISIC2018}~\cite{tschandl2018ham10000,codella2019skin}: The ISIC2018 challenge dataset (task 3) contains 7 categories and we classify the NV (nevus) category as normal samples, and the rest 6 categories as abnormal data. The training set includes 6705 healthy images and 3310 unhealthy images from 6 categories. As the setting of AE-FLOW~\cite{zhao2023aeflow}, the original validation set is used as the test set which includes 123 normal images and 70 abnormal images.
    \item \textbf{Chest X-rays}~\cite{kermany2018identifying}: The chest X-ray dataset contains 1349 normal images and 3883 lesion images in the training set. There are 234 normal images and 390 anomalous images which are diagnosed as pneumonia in the test set.
    \item \textbf{Br35H}\cite{br35h,zhou2024anomalyclip}: Brain Tumor Detection dataset contains 1500 non-tumorous images and  1500 tumorous images. The training set is split with 1200 normal and 1200 anomalous images while the rest 300 normal and anomalous images are used in the test set.
    \item \textbf{OCT}~\cite{kermany2018identifying}: Retinal optical coherence tomography (OCT) contains 83484 images in the training set with 26,315 normal images and 57,169 anomalous images. The images are separated into 4 categories: normal, drusen, CNV, and DME. The images of category normal are considered as normal while the other three types are anomalous. In the test set, each category has 242 images.
\end{itemize}

\subsection{Industrial Datasets}
\begin{itemize}
    \item \textbf{MVTec AD}~\cite{bergmann2019mvtec} is an industrial dataset, containing 15 classes with 5 textures and 10 objects with a total of 3629 normal images in the training set and 1725 normal and anomalous images in the test set. The widely used detection dataset has pixel-level annotations for anomalous images in the test set for evaluating the performance of anomaly localization.
    \item \textbf{Visa}~\cite{zou2022visa} is an industrial dataset comprised of 12 classes exhibiting a wide range of scales and anomaly types. It consists of 10,821 high-resolution images, categorized into 9,621 normal instances and 1,200 anomalous instances.
\end{itemize}

\begin{figure}[htbp]
    \centering
    \centering 
    \vskip -0.5cm
\subfigure[normal images]{%
    \includegraphics[width=0.45\linewidth]{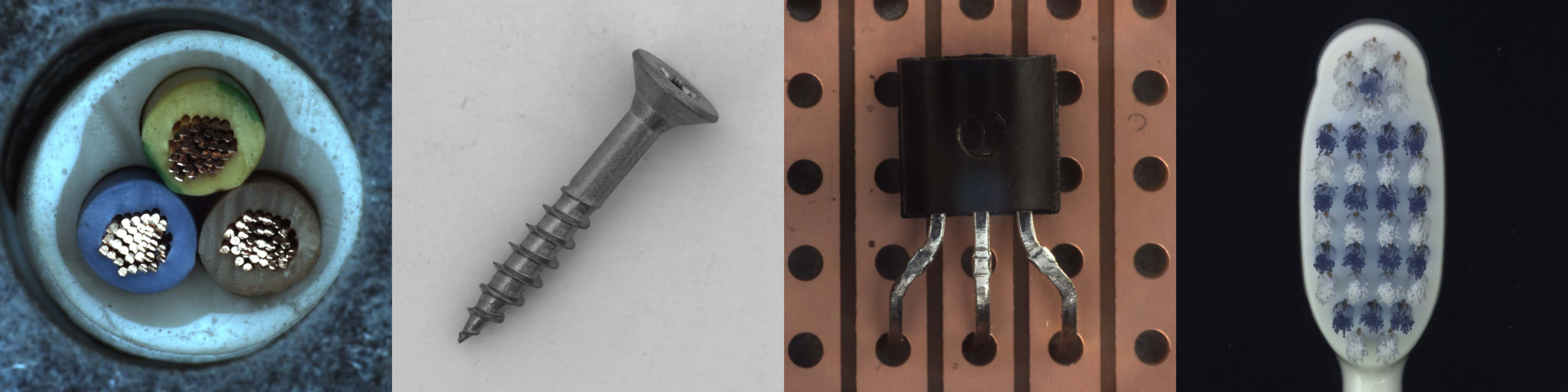}%
}
\hspace{0.3cm}
\subfigure[elastic transformation]{%
    \includegraphics[width=0.45\linewidth]{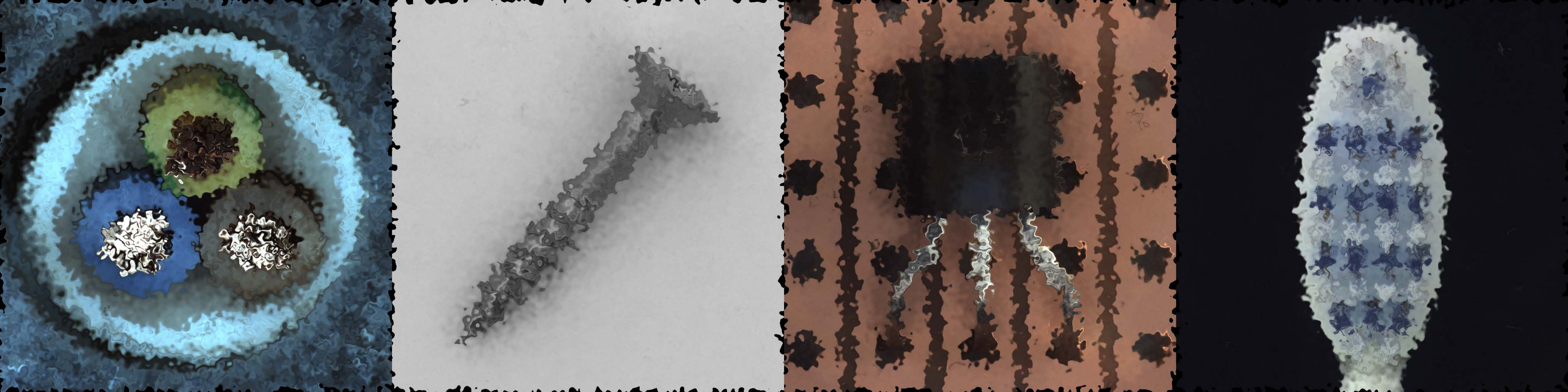}%
}
\caption{Visualization of elastic transformation of images.}
\label{fig:elastic}
\end{figure}

As there are no training anomalous images in the industrial datasets, so as Figure \ref{fig:elastic} shows, we use elastic transformation, which is implemented in PyTorch, to distort the original image content and produce the corresponding anomalous samples, then utilize such anomalies for adversarial training to evaluate the performance of the proposed weakly-supervised method. 

\section{Complete Implementation Details}
\label{appendix:implementation_details}
For the auto-encoder, to encode multiple features into a low-dimensional representation, we first utilize a sequence of $3\times 3$ convolution operators followed by instance normalization and nonlinear activation ReLU to downsample the large feature maps to the same size as the smallest feature map. Then, we concatenate all the downsampled feature maps along channels and feed the concatenated features into residual convolutional blocks used in the backbone ResNet to yield the latent representation $z$. In the decoding process, we employ the residual deconvolution which is composed of $2 \times 2$ transposed convolutions to upsample it stage by stage, leading to a list of reconstructed feature maps. The patch-level energy-based discriminator is implemented by multi-layer perception, which is implemented by $1\times 1$ convolution with stride 1 and nonlinear activation LeakyReLU. 

All medical images and industrial images are resized to 256, $\beta$ and $\gamma$ are set to 0.5, $\lambda$ is set to 0.02. Adam optimizer is utilized with $\beta=(0.5, 0.999)$. The learning rate for the medical dataset is set to 1e-03 and for the industrial dataset is set to 5e-03. The learning rate for the energy-based discriminator is set to 1e-04. For medical datasets, as the previous weakly-supervised method DevNet~\cite{pang2021devnet}, DRA~\cite{ding2022dra} done, Resnet18 is chosen as the pre-trained encoder to extract two-level features from layers 2 and 3. For the industrial dataset, as method RD4AD\cite{deng2022RD}, PatchCore\cite{roth2022patchcore}, WideResnet50 is chosen as the pre-trained encoder to extract features from layers 1, 2, and 3. For the image-level anomaly score, the average of the top k value is the top 100 score in the anomaly score maps. As in the ISIC2018 dataset, nevus is classified as normal samples, similar to other abnormal images, k is chosen as the whole pixel num. The batch size for medical datasets is set to 64 and for industrial datasets is set to 16. The training epoch for the medical dataset is set to 200 and set to 400 for industrial datasets. Our codes are implemented with Python 3.7 and PyTorch 1.13.0 cuda 11.7. Experiments are run on single NVIDIA GeForce RTX3090 GPUs (memory used is less than 24GB).

\section{Additional Experiments Results}

\begin{table*}[htbp]
\small
\centering
{\begin{tabular}{c|ccc|ccc}
\toprule
\multirow{2}{*}{$r_l$} & \multicolumn{3}{c|}{ISIC2018} & \multicolumn{3}{c}{Chest X-ray} \\
\cmidrule(r){2-4}  \cmidrule(r){5-7}
& AUC & F1 & ACC &  AUC & F1 & ACC \\
\cmidrule(r){1-1} \cmidrule(r){2-4}  \cmidrule(r){5-7}
0\% & 87.51 $\pm$ 0.07 & 79.50 $\pm$ 0.68 & 82.90 $\pm$ 0.85 & 85.59 $\pm$ 0.16 & 83.70 $\pm$ 0.30 & 78.04 $\pm$ 1.19 \\
\cmidrule(r){1-1} \cmidrule(r){2-4}  \cmidrule(r){5-7}
1\% & 91.02 $\pm$ 0.10 & 82.04 $\pm$ 0.41 & 85.84 $\pm$ 0.32 & 93.22$\pm$ 0.07 & 88.77 $\pm$ 0.11 & 86.22 $\pm$ 0.34 \\
\cmidrule(r){1-1} \cmidrule(r){2-4}  \cmidrule(r){5-7}
2\% & 91.68 $\pm$ 0.21 & 82.86 $\pm$ 0.24 & 86.87 $\pm$ 0.92 & 95.23 $\pm$ 0.06 & 91.17 $\pm$ 0.35 & 88.94 $\pm$ 0.13 \\
\cmidrule(r){1-1} \cmidrule(r){2-4}  \cmidrule(r){5-7}
5\% & 92.66 $\pm$ 0.10 & 84.00 $\pm$ 0.30 & 87.65 $\pm$ 0.31 & 96.57 $\pm$ 0.03 & 93.87 $\pm$ 0.48 & 92.47 $\pm$ 0.53\\
\midrule
\midrule
\multirow{2}{*}{$r_l$} & \multicolumn{3}{c|}{Br35H} & \multicolumn{3}{c}{OCT} \\
\cmidrule(r){2-4}  \cmidrule(r){5-7}
& AUC & F1 & ACC &  AUC & F1 & ACC \\
\cmidrule(r){1-1} \cmidrule(r){2-4}  \cmidrule(r){5-7}
0\% & 97.37 $\pm$ 0.04 & 96.23 $\pm$ 0.08 & 96.16 $\pm$ 0.08 &  99.33 $\pm$ 0.05 &  97.84 $\pm$ 0.07 & 96.80 $\pm$ 0.09\\
\cmidrule(r){1-1} \cmidrule(r){2-4}  \cmidrule(r){5-7}
1\% & 98.89 $\pm$ 0.07 & 97.88 $\pm$ 0.08 & 97.83 $\pm$ 0.08 & 99.88 $\pm$ 0.07 & 99.33 $\pm$ 0.24 & 99.00 $\pm$ 0.48 \\
\cmidrule(r){1-1} \cmidrule(r){2-4}  \cmidrule(r){5-7}
2\% & 99.07 $\pm$ 0.02 & 98.03 $\pm$ 0.15 & 98.00 $\pm$ 0.13 & 99.91 $\pm$ 0.07 & 99.42 $\pm$ 0.28 & 99.13 $\pm$ 0.42 \\
\cmidrule(r){1-1} \cmidrule(r){2-4}  \cmidrule(r){5-7}
5\% & 99.79 $\pm$ 0.08 & 98.20 $\pm$ 0.08 & 98.17 $\pm$ 0.08 & 99.93 $\pm$ 0.04 & 99.55 $\pm$ 0.18 & 99.33 $\pm$ 0.26 \\
\bottomrule
\end{tabular}}
\caption{Performance of Average AUC(\%), best F1(\%)) and ACC(\%) with standard deviation on medical datasets. With $k=1$, the ratio of labeled anomalies $r_l$ in the training is increased from $0\%$ to $5\%$.}
\label{appendix_tab:medical}
\end{table*}

\begin{table*}[htbp]
    \centering
    \setlength{\tabcolsep}{1mm}
    {\begin{tabular}{c|c|ccccccccccccc}
    \toprule
        & & airplane & automobile & bird & cat & deer & dog & frog & horse & ship & truck & mean \\
        \midrule
        \multirow{2}{*}{Unsupervised} & RD & 88.42 & 91.80 & 76.16 & 72.82 & 85.10 & 74.89 & 89.56 & 87.76 & 90.54 & 90.94 & 85.79 \\
        & Recon(base) & 89.80 & 94.48 & 80.70 & 73.96 & 85.04 & 88.08 & 92.27 & 92.64 & 93.4 & 92.87 & 88.32 \\
        \midrule
        \multirow{5}{*}{\shortstack{Weakly \\{Supervised}}} & Deep SAD & 84.98 & 89.62 & 74.86 & 76.63 & 80.65 & 80.28 & 88.66 & 79.61 & 90.01 & 86.49 & 83.18 \\
         & AA-BiGAN & 83.12 & 90.81 & 80.51 & 73.92 & 75.70 & 86.46 & 91.15 & 81.85 & 87.70 & 83.06 & 83.42 \\
         & DRA & 86.93 & 92.31 & 82.17 & 78.80 & 85.41 & 89.48 & 88.60 & 88.59 & 89.94 & 90.94 & 87.31 \\
         & DevNet & 89.27 & 94.52 & 84.54 & \textbf{84.16} & \textbf{90.35} & 89.30 & 94.65 & 85.83 & 91.59 & 93.33 & 89.75 \\
         & Ours & \textbf{93.21} & \textbf{96.71} & \textbf{87.37} & 81.76 & 90.00 & \textbf{91.08} & \textbf{95.21} & \textbf{95.14} & \textbf{95.95} & \textbf{95.02} & \textbf{92.14} \\
    \bottomrule
    \end{tabular}}
    \caption{One class detection performance of average AUC(\%) on CIFAR10. Weakly supervised methods are run under the setting $r_l=5\%, k=3$.}
    \label{tab:appendix_cifar10}
\end{table*}

\begin{table*}[htbp]
    \centering
    {
    \begin{tabular}{cccccccccccccccccc}
    \toprule
        \multirow{2}{*}{$r_l$} & \multirow{2}{*}{Normalization} & \multicolumn{3}{c}{ISIC2018} & \multicolumn{3}{c}{Chest X-Ray} & \multicolumn{3}{c}{Br35H} \\
        \cmidrule(lr){3-5} \cmidrule(lr){6-8}  \cmidrule(lr){9-11}  \cmidrule(lr){12-14} 
        & & AUC & F1 & ACC & AUC & F1 & ACC & AUC & F1 & ACC \\
       \cmidrule(r){1-1} \cmidrule(r){2-2} \cmidrule(lr){3-5} \cmidrule(lr){6-8}  \cmidrule(lr){9-11}  \cmidrule(lr){12-14} 
       \multirow{2}{*}{0\%} & Batch & \textbf{87.56} & 79.27 & 82.38 & 85.23 & 83.21 & 77.24 & \textbf{97.94} & \textbf{96.58} & \textbf{96.50}\\
       & Instance & 87.51 & \textbf{79.50} & \textbf{82.90} & \textbf{85.59} & \textbf{83.70} & \textbf{78.04} & 97.37 & 96.23 & 96.16 \\
       \midrule
       \multirow{2}{*}{5\%} & Batch & 90.46 & 82.58 & 86.53 & 94.21 & 89.60 & 87.50 & 99.20 & 97.24 & 97.17 \\
       & Instance & \textbf{92.66} & \textbf{84.00} & \textbf{87.65} & \textbf{96.57} & \textbf{93.87} & \textbf{92.47} & \textbf{99.79} & \textbf{98.20} & \textbf{98.17} \\
       \bottomrule
    \end{tabular}}
     \caption{Performance of different normalizing ways when using anomalies.
    }
    \label{tab:normalize}
\end{table*}

\begin{figure*}[htbp]
  \centering 
    \includegraphics[width=0.9\textwidth]{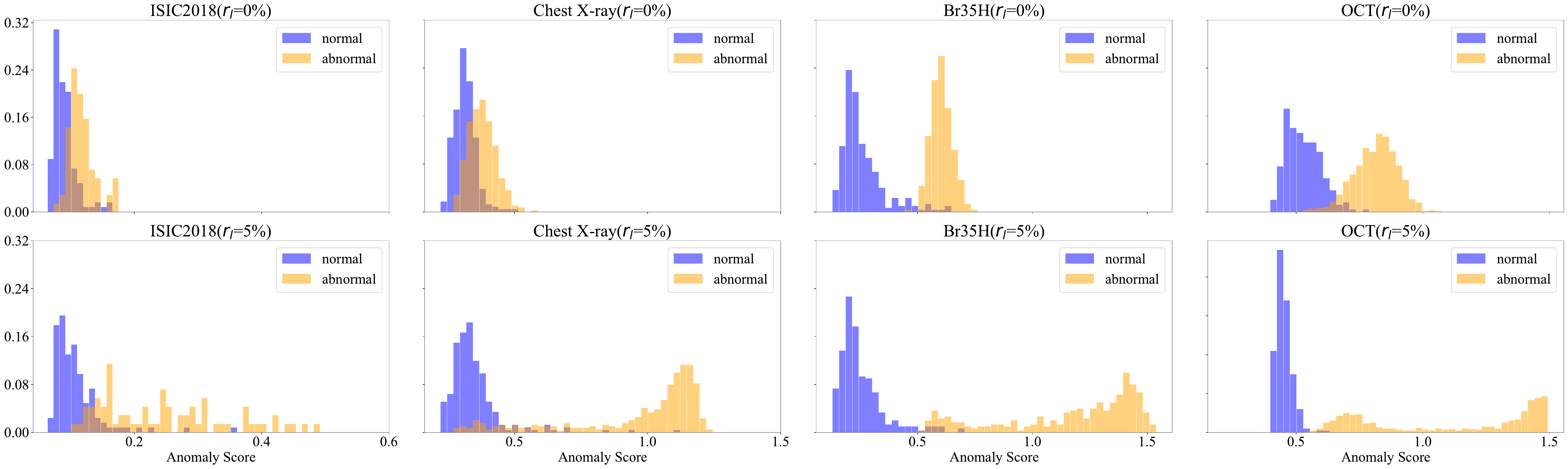}
    \captionof{figure}{The two rows separately exhibit anomaly score distribution on the test set trained without/with $5\%$ anomalies.} 
    \label{fig:anomaly_score} 
\end{figure*}

\label{appendix:experiment}
\subsection{Detailed Experimental Results}
\label{appendix:detailed_experiment}
In addition to the averaged performances shown in the main content, we report the average performance with the standard deviation of three runs with different random seeds in Table \ref{appendix_tab:medical}, \ref{tab:appendix_mvtec} and \ref{tab:appendix_visa}.

\subsection{Anomaly Scores Distributions of Reconstruction and Proposed Methods}

Figure~\ref{fig:anomaly_score} illustrates the distributions of anomaly scores on the test set of the reconstruction-based model and our model trained with 5\% anomalous samples in the training set. The results clearly show that the proposed method can increase the disparity of anomaly scores between normal and abnormal samples, proving that our knowledge-aware adversarial training strategy can effectively regularize the reconstruction of anomalous samples and give larger anomaly scores to anomalous samples. 

\subsection{Different Normalization Way}
During the training stage, the collected anomalous samples will also pass through the auto-encoder. If using usual batch normalization(BN),  the inclusion of anomalies may cause a considerable shift to the mean and variance of BN compared to only using normal samples, negatively affecting the reconstruction quality of normal samples. Thus, instance normalization is uniformly employed in our network, where the mean and variance are calculated for every sample.

To evaluate the influence of normalization way when introducing anomalies, we mainly conduct the ablation study on three medical datasets with the setting $k=1$, $r_l=0\%$, and $r_l=5\%$. As shown in Table \ref{tab:normalize}, there is a small difference between batch and instance normalization when only using normal samples during training. But when introducing anomalies in our methods, although using batch normalization can also improve the detection performance, the improvement of batch normalization is much smaller than instance normalization, showing that it can negatively influence our energy-based strategy to utilize anomalies, so instance normalization is uniformly employed in our model.

\subsection{Qualitative Visualization Results}
We show the visualization results of using $5\%$ of anomalies in the training set. As shown in Figure \ref{fig:appendix_medical} and \ref{fig:appendix_mvtec}, the anomalous parts in the medical and industrial datasets are assigned higher anomaly scores. The qualitative results in these figures show that our methods can be used to localize anomalous areas using incomplete anomalous information.
 
\subsection{Additional Experimental Results}
We mainly evaluate the performance of our method on real-world anomaly detection datasets including medical and industrial datasets in previous experiments. Many anomaly detection methods conduct experiments on the commonly used dataset CIFAR10 for one-class classification. We also conduct experiments on CIFAR10. For each category in CIFAR10, we regard the category as normal while other categories as anomalous classes. Following the data setting on the medical dataset, for each class in CIFAR10, we conduct 5 random experiments by using 5\% anomalies from randomly selected 3 anomalous classes in the training set. Totally, we conduct 50 experiments in the weakly-supervised scene. The results of performance are shown in Table \ref{tab:appendix_cifar10}. Compared to the reconstruction base, our method improves the average detection AUC on CIFAR10 from 88.32\% to 92.14\%, proving that our method is effective in such semantic anomaly detection datasets. Compared to other weakly-supervised methods, our method also gets the overall best detection performance for the one-class anomaly detection in CIFAR10.

\begin{table*}[ht]
\centering
\resizebox{0.98\textwidth}{!}{\begin{tabular}{c|cccccccc}
\toprule
Method & DRAEM	& RD4AD	& PatchCore	& RD++ & DMAD & D3AD & CKAAD (ours) \\
\midrule
carpet      &(97.0, 95.5, 92.9) & (98.9, 98.9, 97.0) & (98.7, 99.0, 96.6) & (100, 99.2, 97.7) & (100, 99.1, 86.1) & (94.2, 97.6, 95.1) & (100$\pm$0.00, 99.3$\pm$0.01, 97.9$\pm$0.02) \\	
grid        &(99.9, 99.7, 98.4) & (100, 99.3, 97.6) & (98.2, 98.7, 96.0) & (100, 99.3, 97.7) & (100, 99.2, 72.4) & (100, 99.2, 96.9) & (100$\pm$0.00, 99.3$\pm$0.02, 97.6$\pm$0.03) \\	
leather	    &(100, 98.6, 98.0) & (100, 99.4, 99.1) & (100, 99.3, 98.9) & (100, 99.4, 99.2) & (100, 99.5, 97.7) & (98.5, 99.4, 98.1) & (100$\pm$0.00, 99.5$\pm$0.01, 99.2$\pm$0.01) \\	
tile        &(99.6, 99.2, 98.9) & (99.3, 95.6, 90.6) & (98.7, 95.4, 87.3) & (99.7, 96.6, 92.4) & (100, 96.0, 82.7) & (95.5, 94.7, 93.6) & (99.9$\pm$0.06, 95.8$\pm$0.02, 91.6$\pm$0.13) \\	
wood	    &(99.1, 96.4, 94.6) & (99.2, 95.3, 90.9) & (99.2, 95.0, 89.4) & (99.3, 95.8, 93.3) & (100, 95.5, 86.3) & (99.7, 95.9, 91.0) & (99.6$\pm$0.07, 95.3$\pm$0.04, 92.2$\pm$0.16) \\	
bottle	    &(99.2, 99.1, 97.2) & (100, 98.7, 96.6) & (100, 98.6, 96.2) & (100, 98.8, 97.0) & (100, 98.9, 96.0) & (100, 98.6, 96.0) & (100$\pm$0.00, 99.0$\pm$0.02, 97.3$\pm$0.05) \\	
cable	    &(91.8, 94.7, 76.0) & (95.0, 97.4, 91.0) & (99.5, 98.4, 92.5) & (99.2, 98.4, 93.9) & (99.1, 98.1, 95.1) & (97.8, 93.3, 87.3) & (99.4$\pm$0.11, 98.3$\pm$0.15, 93.9$\pm$0.34) \\	
capsule	    &(98.5, 94.3, 91.7) & (96.3, 98.7, 95.8) & (98.1, 98.8, 95.5) & (99.0, 98.8, 96.4) & (98.9, 98.3, 89.7) & (96.6, 97.9, 90.7) & (97.5$\pm$0.02, 98.8$\pm$0.04, 96.3$\pm$0.04) \\	
hazelnut	&(100, 99.7, 98.1) & (99.9, 98.9, 95.5) & (100, 98.7, 93.8) & (100, 99.2, 96.3) & (100, 99.1, 96.4) & (98.0, 98.8, 91.8) & (100$\pm$0.00, 99.2$\pm$0.04, 96.2$\pm$0.04) \\	
metal\_nut	&(98.7, 99.5, 94.1) & (100, 97.3, 92.3) & (100, 98.4, 91.4) & (100, 98.1, 93.0) & (100, 97.7, 94.6) & (98.9, 96.1, 89.7) & (100$\pm$0.00, 98.1$\pm$0.04, 93.4$\pm$0.19) \\	
pill	    &(98.9, 97.6, 88.9) & (96.6, 98.2, 96.4) & (96.6, 97.4, 93.2) & (98.4, 98.3, 97.0) & (97.3, 98.7, 95.0) & (99.2, 98.2, 96.2) & (98.1$\pm$0.24, 98.3$\pm$0.12, 96.8$\pm$0.04) \\	
screw	    &(93.9, 97.6, 98.2) & (97.0, 99.6, 98.2) & (98.1, 99.4, 97.9) & (98.9, 99.7, 98.6) & (100, 99.6, 94.3) & (83.9, 99.0, 95.5) & (98.9$\pm$0.03, 99.6$\pm$0.04, 98.4$\pm$0.05) \\	
toothbrush	&(100, 98.1, 90.3) & (99.5, 99.1, 94.5) & (100, 98.7, 91.5) & (100, 99.1, 94.2) & (100, 99.4, 92.5) & (100, 99.0, 94.6) & (100$\pm$0.00, 99.2$\pm$0.13, 94.8$\pm$0.34) \\	
transistor	&(93.1, 90.9, 81.6) & (96.7, 92.5, 78.0) & (100, 96.3, 83.7) & (98.5, 94.3, 81.8) & (98.7, 95.4, 86.9) & (96.8, 95.6, 86.9) & (100$\pm$0.00, 97.3$\pm$0.08, 85.6$\pm$0.37) \\	
zipper	    &(100, 98.8, 96.3) & (98.5, 98.2, 95.4) & (99.4, 98.8, 97.1) & (98.6, 98.8, 96.3) & (99.6, 98.3, 92.9) & (98.2, 98.3, 95.3) & (99.3$\pm$0.07, 98.8$\pm$0.05, 96.6$\pm$0.26) \\	
\midrule
average	    &(98.0, 97.3, 93.0) & (98.5, 97.8, 93.9) & (99.1, 98.1, 93.4) & (99.4, 98.3, 95.0) & (\textbf{99.5}, 98.2, 90.6) & (97.2, 97.4, 93.3) & (\textbf{99.5}$\pm$0.04, \textbf{98.4}$\pm$0.05, \textbf{95.2}$\pm$0.14) \\
\bottomrule
\end{tabular}}
\caption{Anomaly detection results in terms of (image AUC, pixel AUC, PRO) of 15 classes on the MVTec dataset. The performance of our proposed method CKAAD is reported over 3 experiments with standard deviation.}
\label{tab:appendix_mvtec}
\end{table*}

\begin{table*}[ht]
\centering
\resizebox{0.98\textwidth}{!}{\begin{tabular}{c|cccccccc}
\toprule
Method & DRAEM	& RD4AD	& PatchCore	& RD++ & DMAD & D3AD & CKAAD (ours) \\
\midrule
candle          &(94.4, 97.3, 93.7) & (92.2, 97.9, 92.2) & (98.6, 99.5, 94.0) & (96.4, 98.6, 93.8) & (92.7, 98.1, 90.6) & (95.6, -, 92.7) & (96.6±0.08, 98.8±0.01, 94.4±0.01) \\	
capsules        &(76.3, 99.1, 84.5) & (90.1, 89.5, 56.9) & (81.6, 99.5, 85.5) & (92.1, 99.4, 95.8) & (88.0, 99.2, 88.4) & (88.5, -, 95.7) & (90.1±0.21, 99.4±0.02, 95.6±0.02) \\	
cashew	        &(90.7, 88.2, 51.8) & (99.6, 95.8, 79.0) & (97.3, 98.9, 94.5) & (97.8, 95.8, 91.2) & (95.0, 95.3, 88.8) & (94.2, -, 89.4) & (98.2±0.05, 97.3±0.11, 93.5±0.11) \\	
chewing\_gum    &(94.2, 97.1, 60.4) & (99.7, 99.0, 92.5) & (99.1, 99.1, 84.6) & (96.4, 99.4, 88.1) & (97.4, 97.9, 73.9) & (99.7, -, 94.1) & (99.2±0.13, 98.3±0.07, 87.6±0.07) \\	
fryum	        &(97.4, 92.7, 93.1) & (96.6, 94.3, 81.0) & (96.2, 93.8, 85.3) & (95.8, 96.5, 90.0) & (98.0, 97.0, 92.2) & (96.5, -, 91.7) & (97.1±0.12, 97.1±0.07, 92.2±0.07) \\	
macaroni1	    &(95.0, 99.7, 96.7) & (98.4, 97.7, 71.3) & (97.5, 99.8, 95.4) & (94.0, 99.7, 96.9) & (94.3, 99.7, 97.1) & (94.3, -, 99.3) & (98.1±0.12, 99.5±0.03, 95.2±0.03) \\	
macaroni2	    &(96.2, 99.9, 99.6) & (97.6, 87.7, 68.0) & (78.1, 99.1, 94.4) & (88.0, 99.7, 97.7) & (90.4, 99.7, 98.5) & (92.5, -, 98.3) & (89.8±0.09, 99.3±0.00, 97.6±0.00) \\	
pcb1	        &(54.8, 90.5, 24.8) & (97.6, 75.0, 43.2) & (98.5, 99.9, 94.3) & (97.0, 99.7, 95.8) & (95.8, 99.8, 96.2) & (97.7, -, 96.4) & (97.5±0.10, 99.8±0.00, 96.4±0.00) \\	
pcb2	        &(77.8, 90.5, 49.4) & (91.1, 64.8, 46.4) & (97.3, 99.0, 89.2) & (97.2, 99.0, 90.6) & (96.9, 99.0, 89.3) & (98.3, -, 94.0) & (97.4±0.02, 98.8±0.05, 92.1±0.05) \\	
pcb3	        &(94.5, 98.6, 89.7) & (95.5, 95.5, 80.3) & (97.9, 99.2, 90.9) & (96.8, 99.2, 93.1) & (98.3, 99.3, 93.6) & (97.4, -, 94.2) & (97.0±0.04, 99.2±0.00, 95.2±0.00) \\	
pcb4	        &(93.4, 88.0, 64.3) & (96.5, 92.8, 72.2) & (99.6, 98.6, 90.1) & (99.8, 98.6, 91.9) & (99.7, 98.8, 91.4) & (99.8, -, 86.4) & (99.9±0.00, 98.5±0.03, 91.2±0.03) \\	
pipe\_fryum	    &(99.4, 90.9, 75.9) & (97.0, 92.0, 68.3) & (99.8, 99.1, 95.7) & (99.6, 99.1, 95.6) & (99.0, 99.3, 95.3) & (96.9, -, 97.2) & (99.9±0.05, 99.2±0.02, 96.9±0.02) \\	
\midrule
average	        &(88.7, 94.4, 73.7) & (96.0, 90.1, 70.9) & (95.1, \textbf{98.8}, 91.2) & (95.9, 98.7, 93.4) & (95.5, 98.6, 91.3) & (96.0, 97.9, \textbf{94.1}) & (\textbf{96.7}±0.08, \textbf{98.8}±0.03, 94.0±0.03)	\\
\bottomrule
\end{tabular}}
\caption{Anomaly detection results in terms of (image AUC, pixel AUC, PRO) of 12 classes on the Visa dataset. The performance of our proposed method CKAAD is reported over 3 experiments with standard deviation.}
\label{tab:appendix_visa}
\end{table*}

\begin{figure}[htbp]
    \centering 
\subfigure[Normal Chest X-ray]{%
    \includegraphics[width=0.48\linewidth]{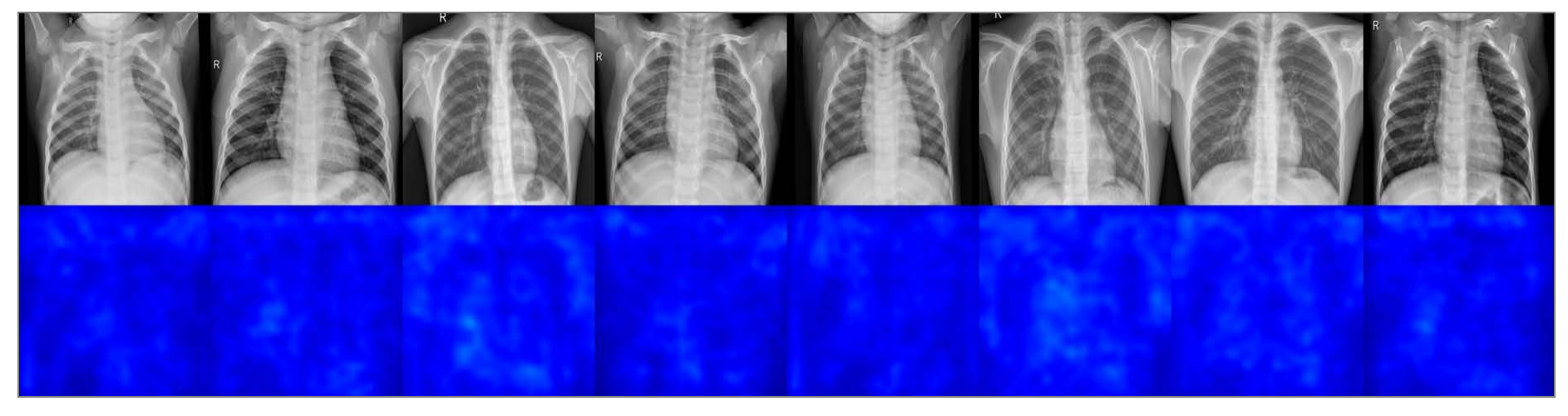}%
}
\subfigure[Abnormal Chest X-ray]{%
    \includegraphics[width=0.48\linewidth]{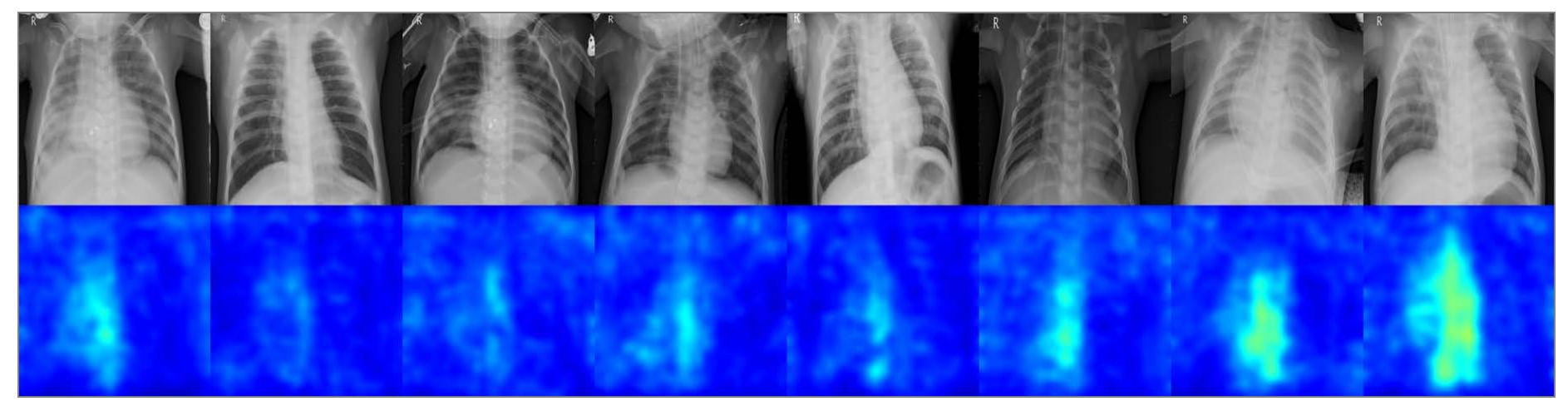}%
}
\subfigure[Normal Br35H]{%
    \includegraphics[width=0.48\linewidth]{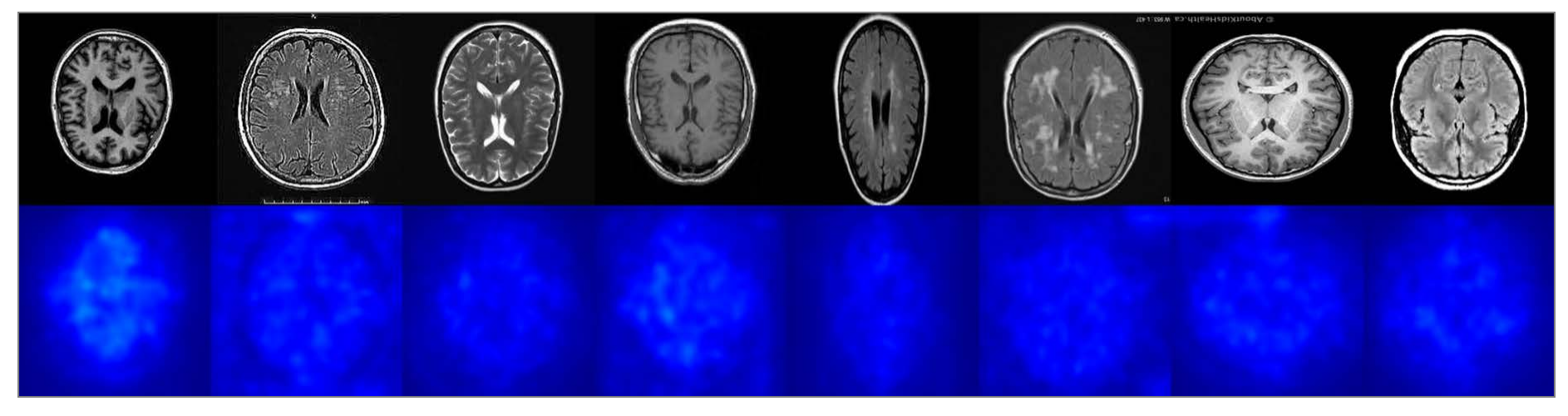}%
}
\subfigure[Abnormal Br35H]{%
    \includegraphics[width=0.48\linewidth]{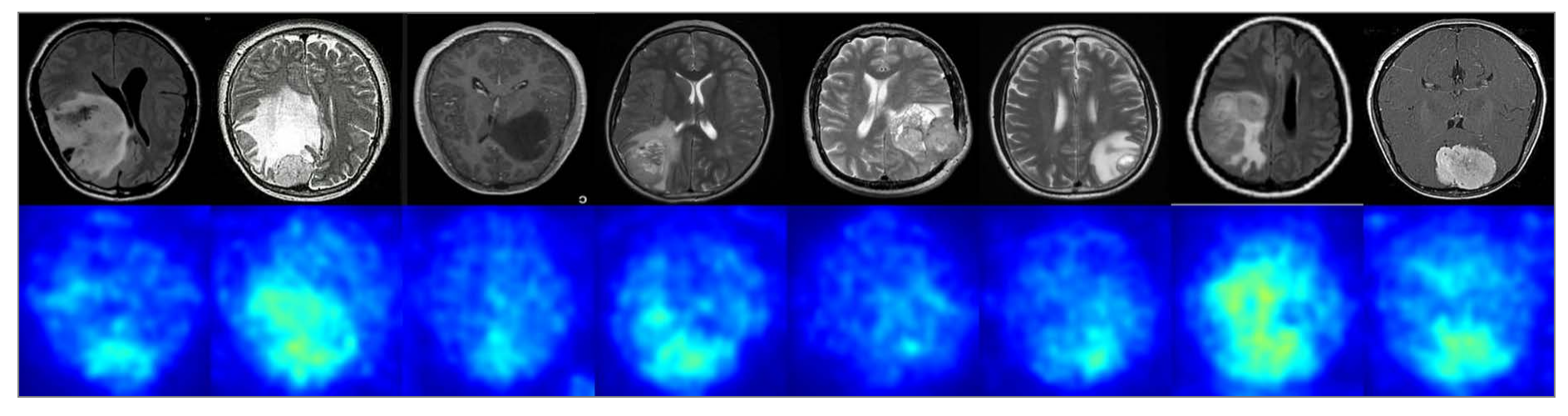}%
}
\subfigure[Normal OCT]{%
    \includegraphics[width=0.48\linewidth]{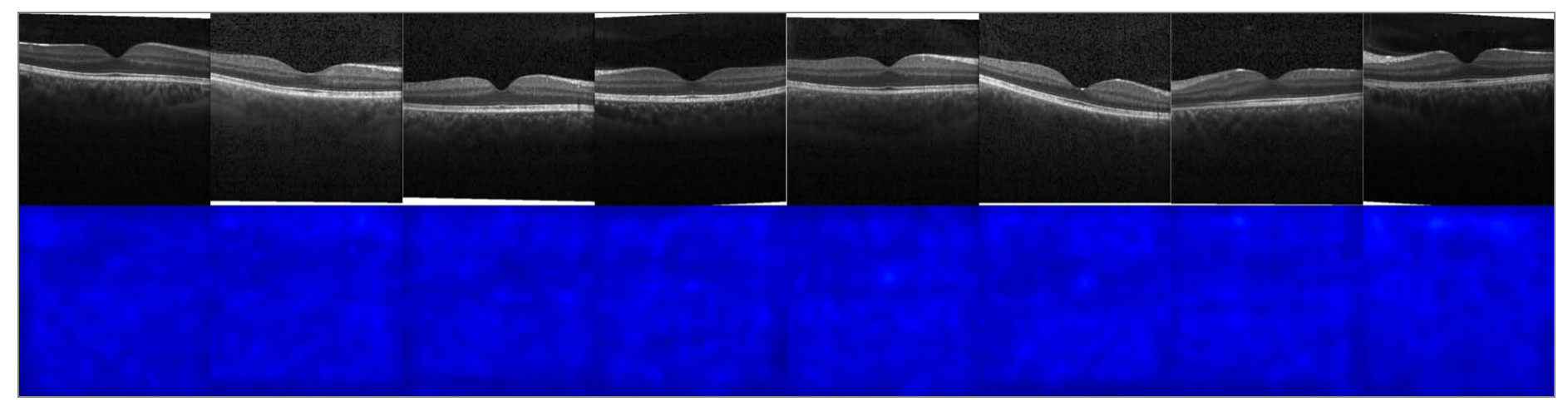}%
}
\subfigure[Abnormal OCT]{%
    \includegraphics[width=0.48\linewidth]{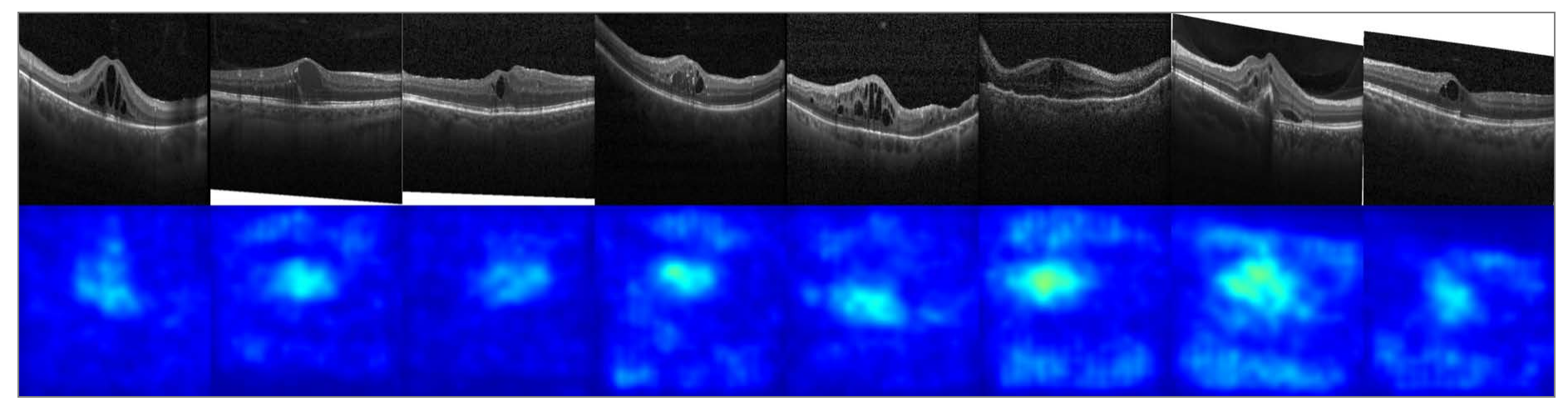}%
}
  
\subfigure[Normal ISIC]{%
    \includegraphics[width=0.48\linewidth]{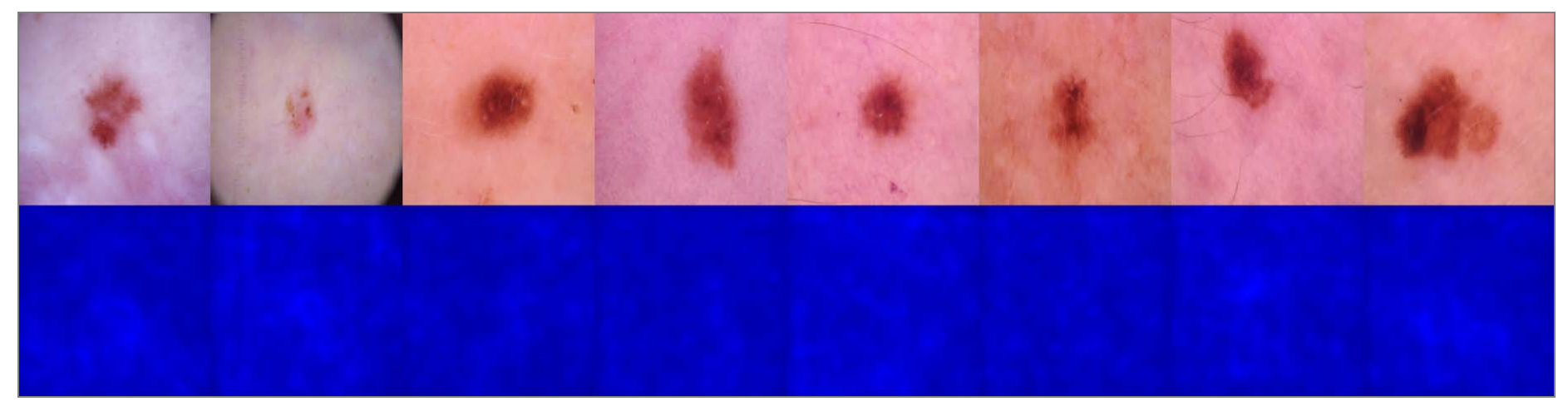}%
}
\subfigure[Abnormal ISIC]{%
    \includegraphics[width=0.48\linewidth]{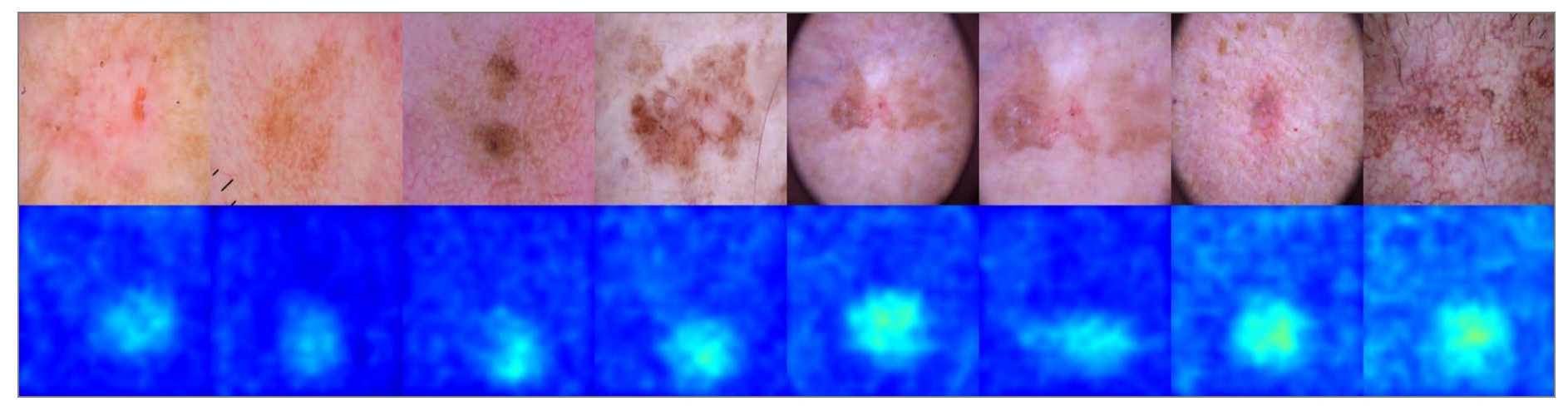}%
}
\caption{Detection visualization on medical datasets. In each subfigure, the two rows are origin images and anomaly score maps $S_{map}$ respectively.}
\label{fig:appendix_medical}
\end{figure}

\begin{figure}[htbp]
\centering 
\subfigure[mvtec: carpet]{%
    \includegraphics[width=0.45\linewidth]{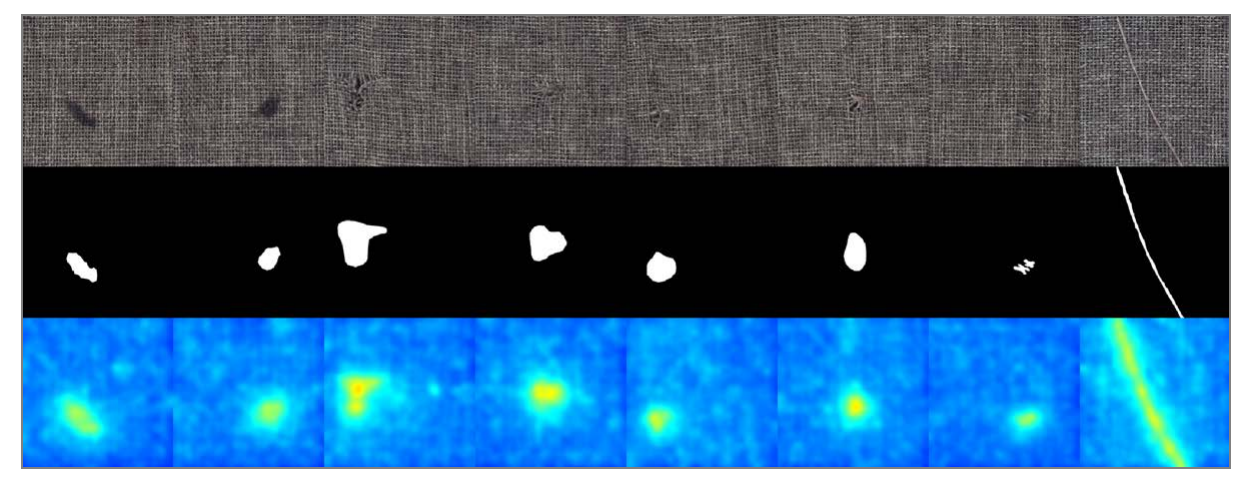}%
}
\subfigure[mvtec: grid]{%
    \includegraphics[width=0.45\linewidth]{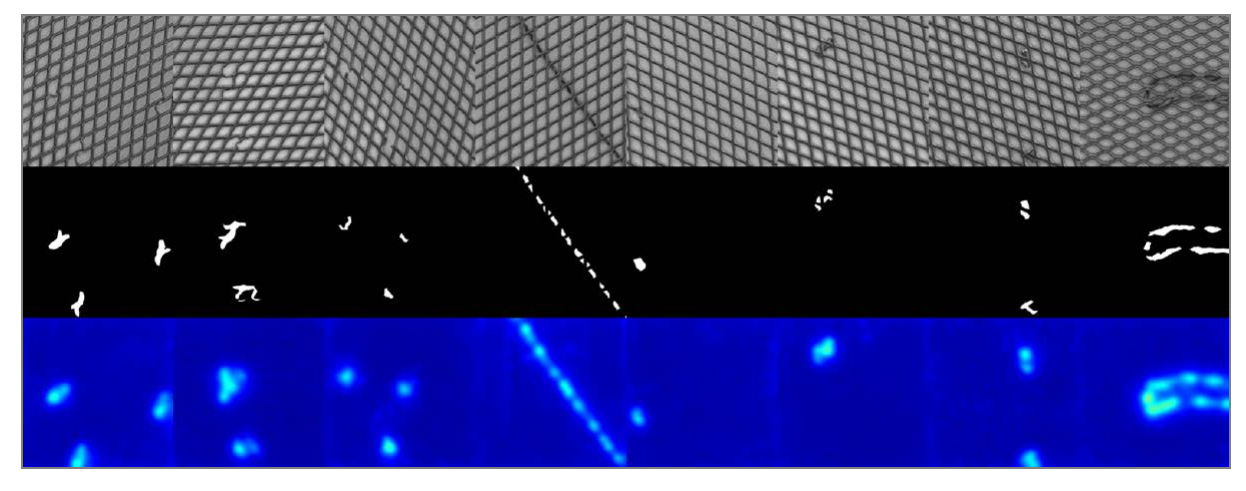}%
}
\subfigure[mvtec: tile]{%
    \includegraphics[width=0.45\linewidth]{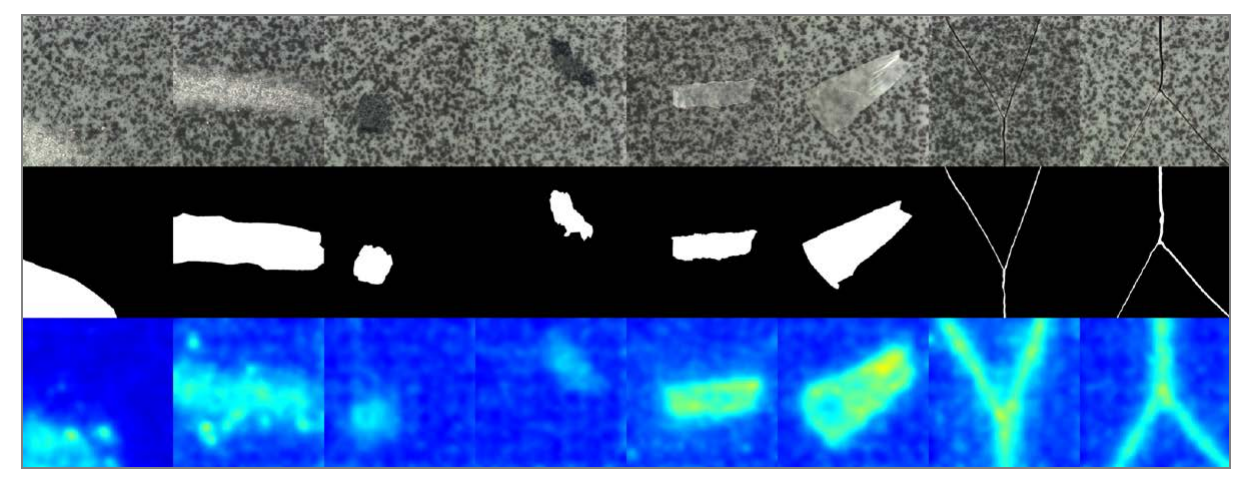}%
}
\subfigure[mvtec: wood]{%
    \includegraphics[width=0.45\linewidth]{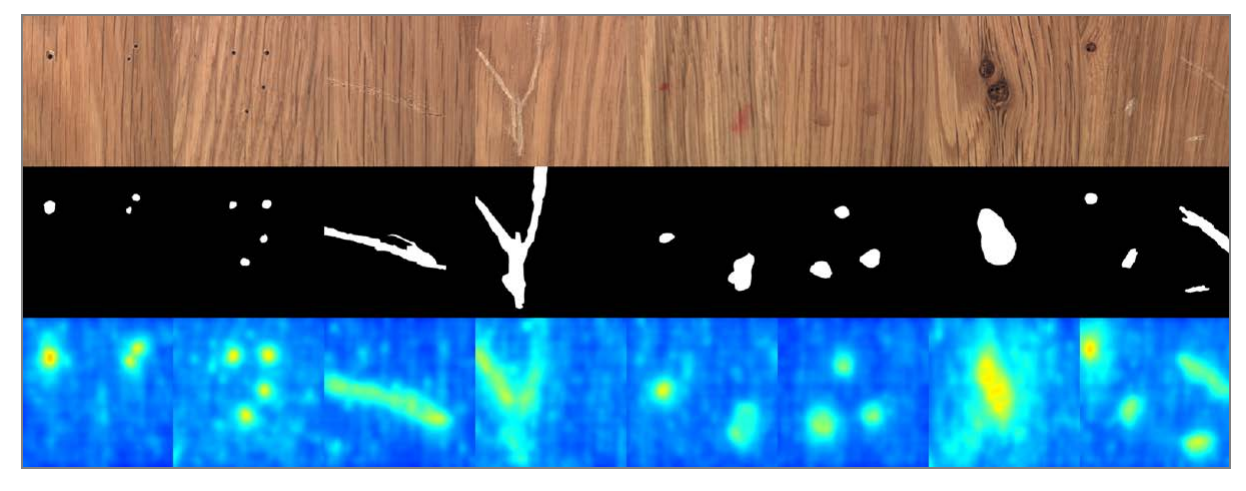}%
}
\subfigure[mvtec: bottle]{%
    \includegraphics[width=0.45\linewidth]{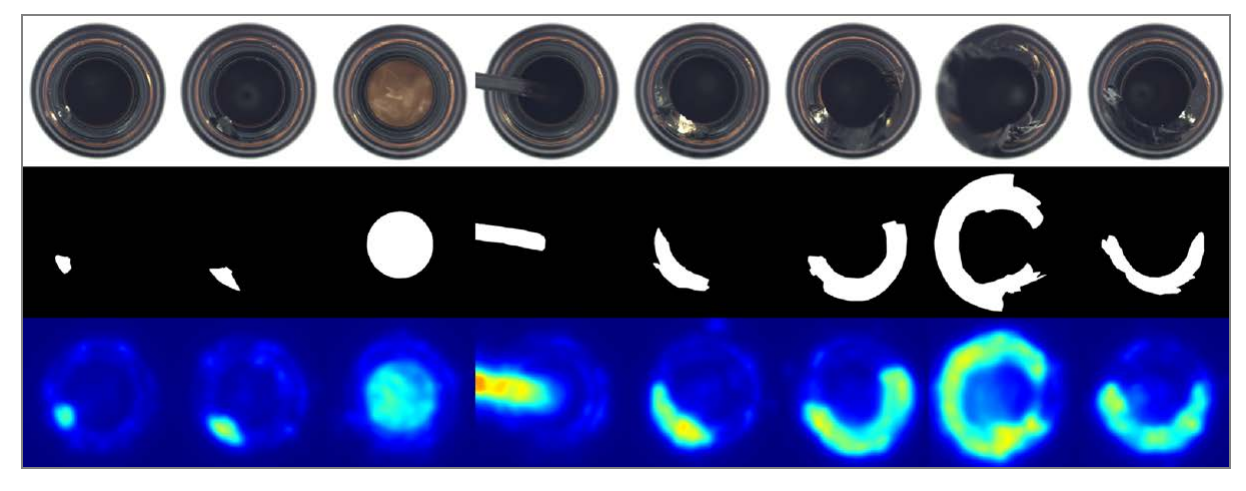}%
} 
\subfigure[mvtec: cable]{%
    \includegraphics[width=0.45\linewidth]{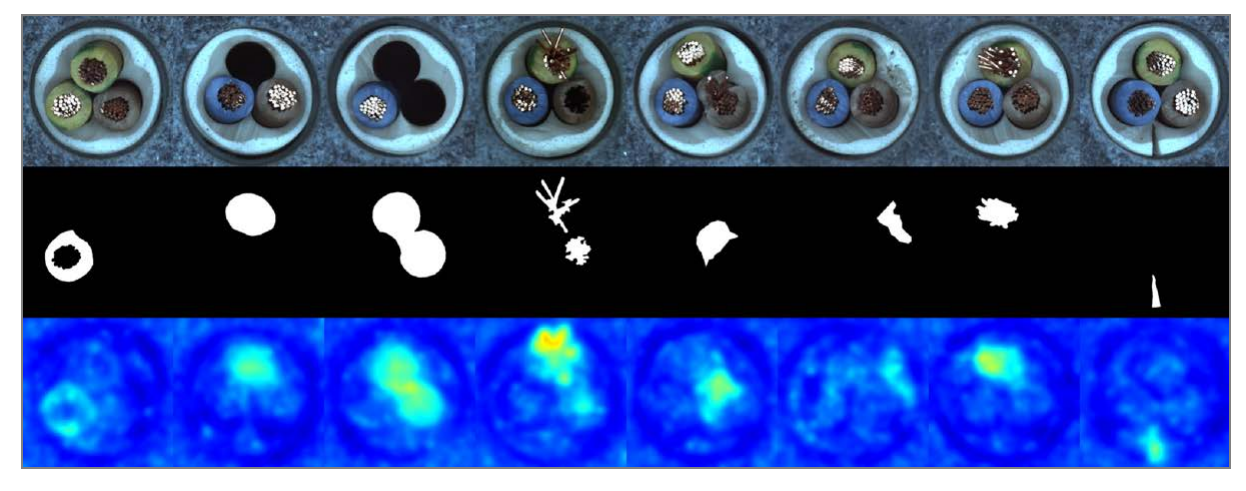}%
}
\subfigure[mvtec: capusle]{%
    \includegraphics[width=0.45\linewidth]{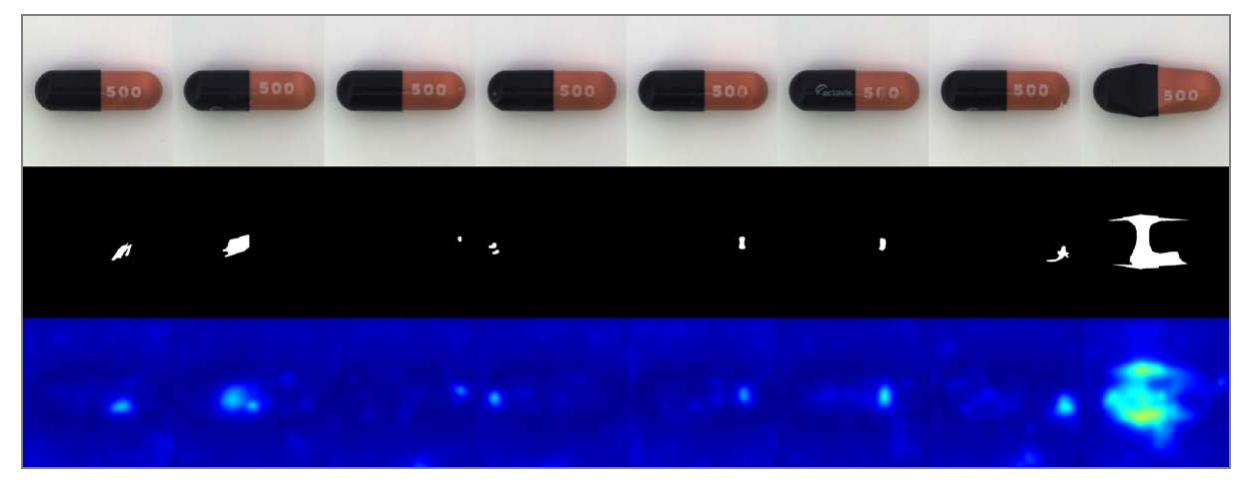}%
}
\subfigure[mvtec: hazelnut]{%
    \includegraphics[width=0.45\linewidth]{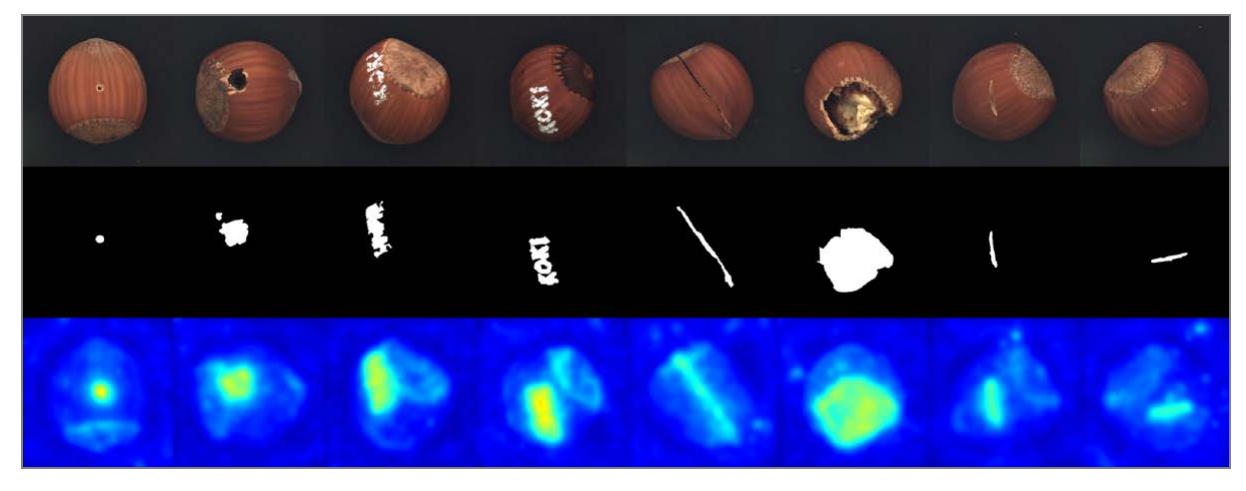}%
}
\subfigure[mvtec: metal\_nut]{%
    \includegraphics[width=0.45\linewidth]{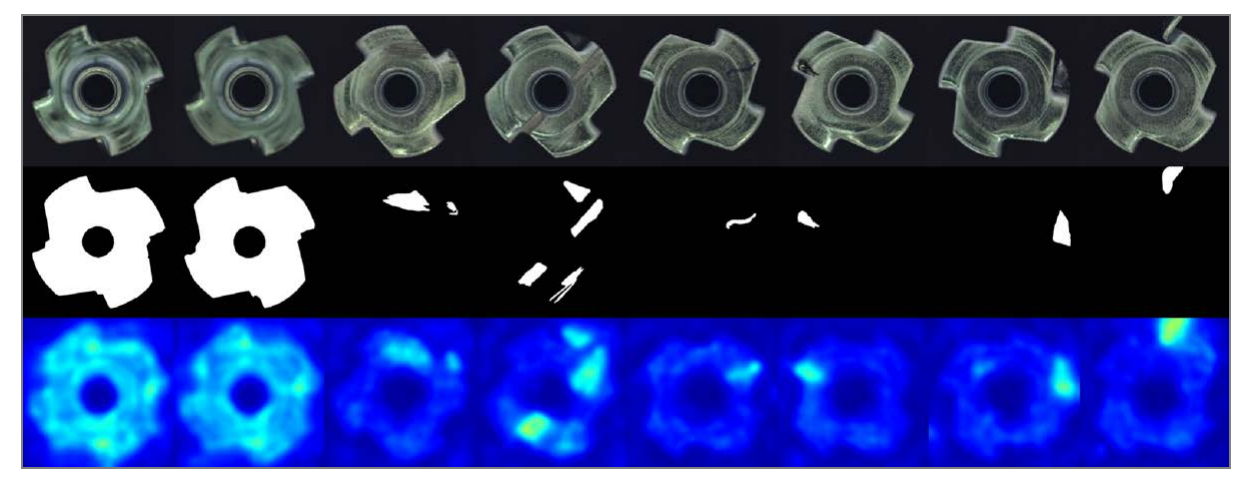}%
}
\subfigure[mvtec: screw]{%
    \includegraphics[width=0.45\linewidth]{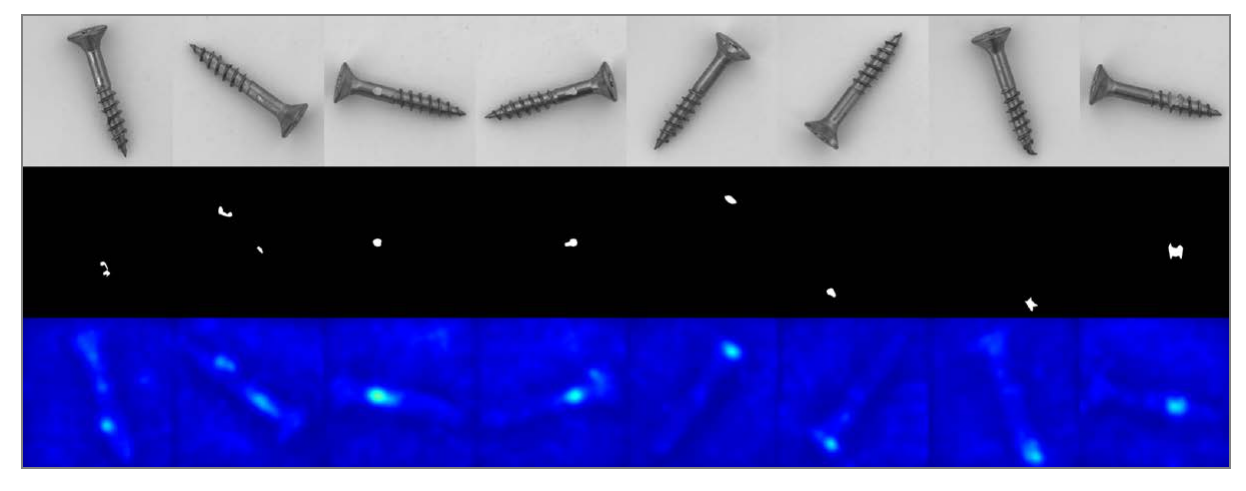}%
}
\subfigure[mvtec: transistor]{%
    \includegraphics[width=0.45\linewidth]{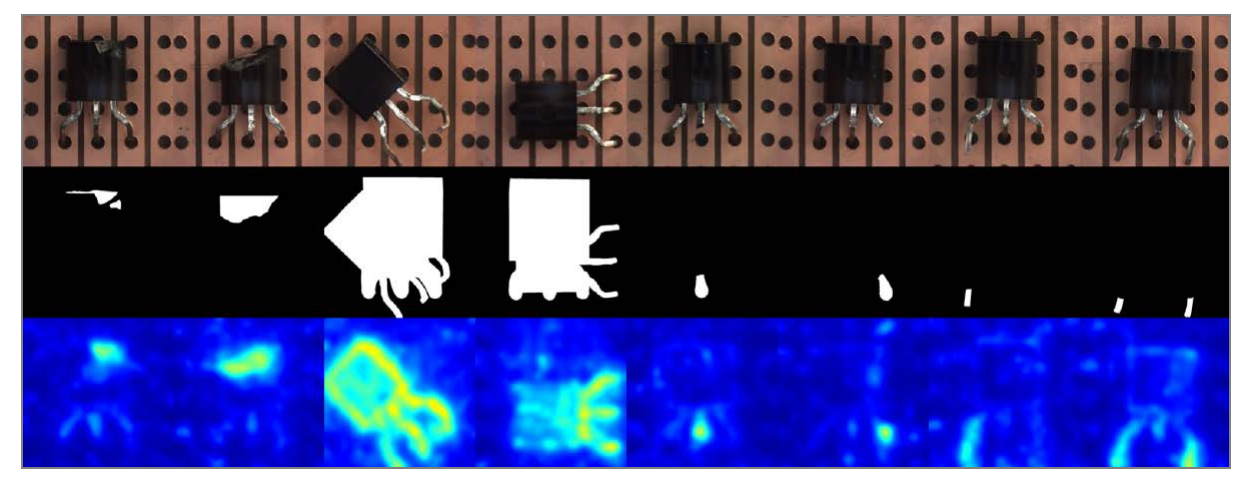}%
}
\subfigure[mvtec: toothbrush]{%
    \includegraphics[width=0.45\linewidth]{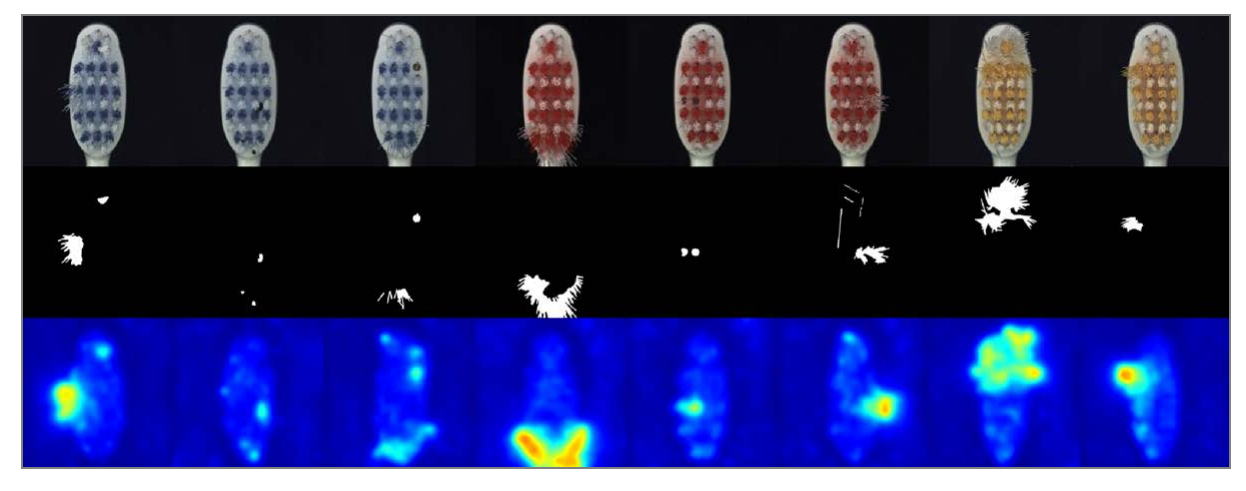}%
}
\subfigure[visa: candle]{%
    \includegraphics[width=0.45\linewidth]{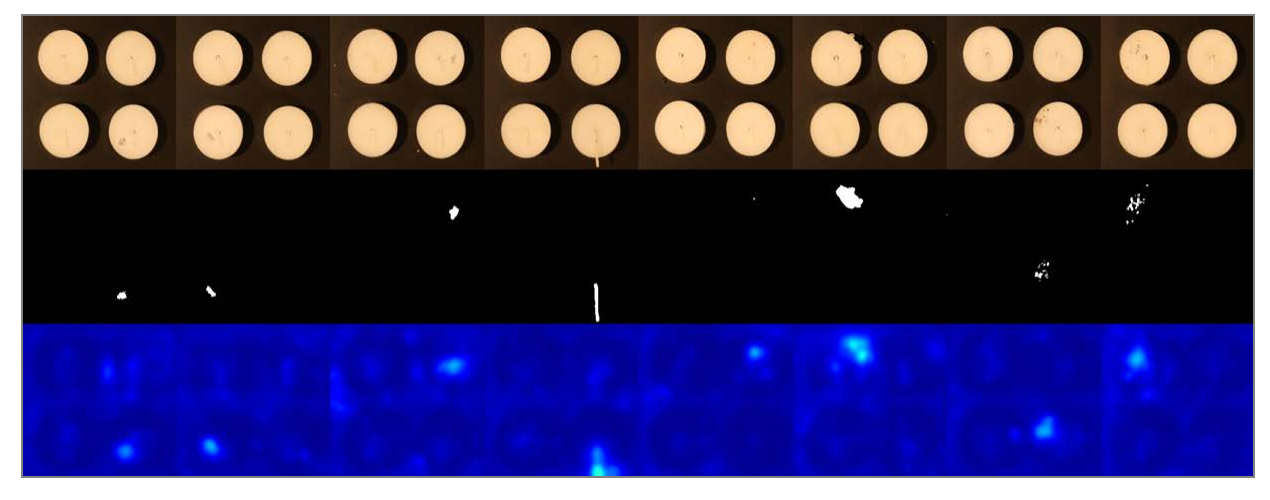}%
}
\subfigure[visa: capsules]{%
    \includegraphics[width=0.45\linewidth]{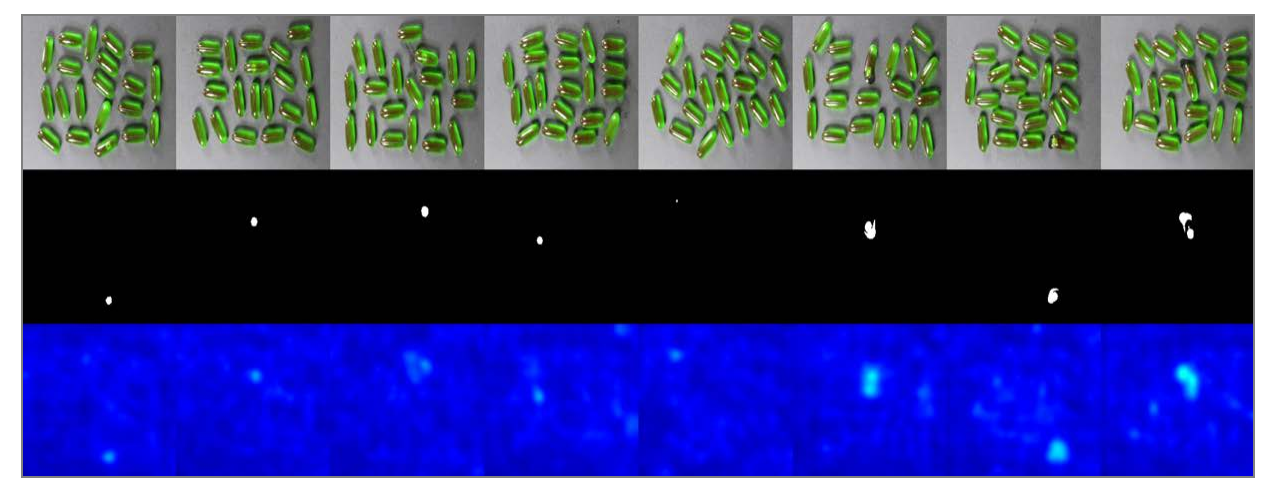}%
}
\subfigure[visa:cashew]{%
    \includegraphics[width=0.45\linewidth]{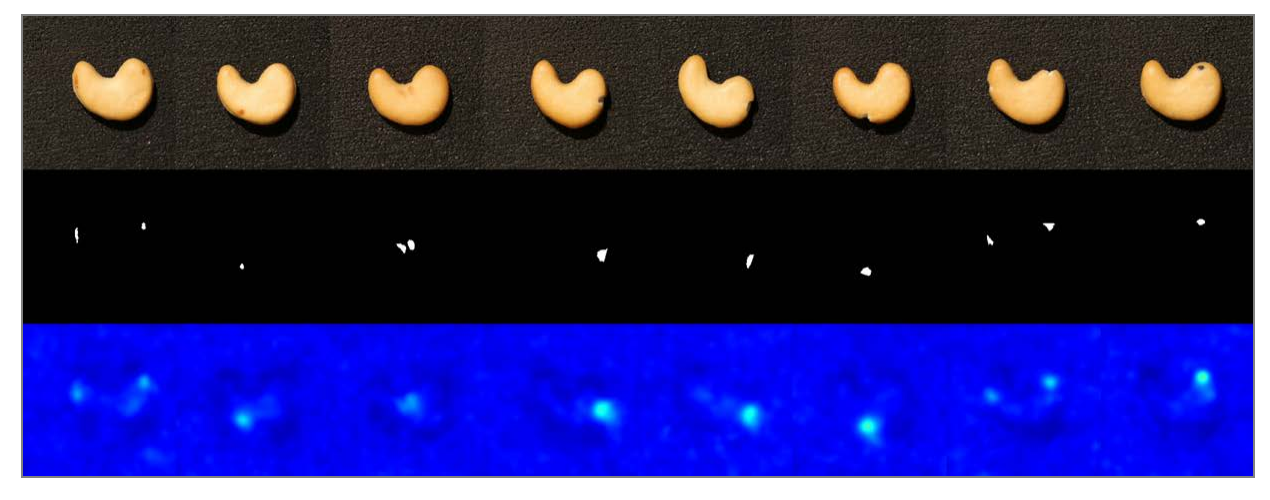}%
}
\subfigure[visa: chewinggum]{%
    \includegraphics[width=0.45\linewidth]{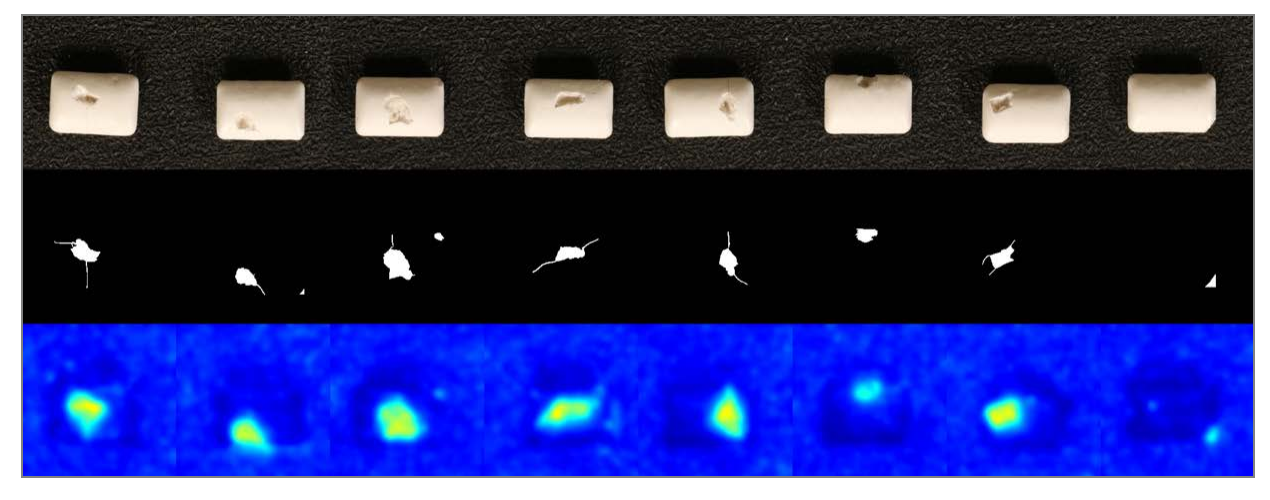}%
}
\subfigure[visa: fryum]{%
    \includegraphics[width=0.45\linewidth]{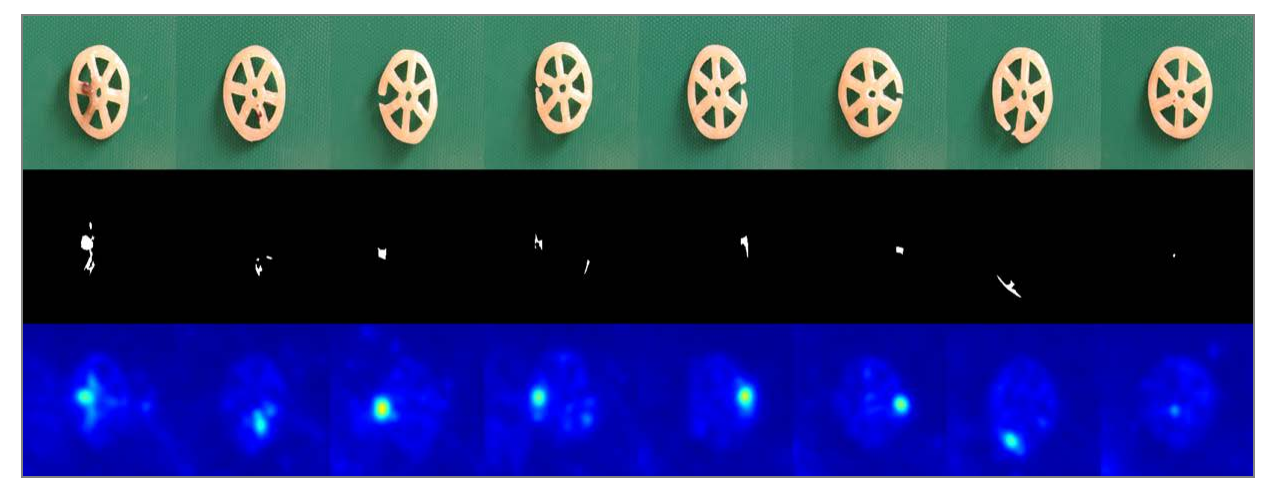}%
}
\subfigure[visa: macaroni1]{%
    \includegraphics[width=0.45\linewidth]{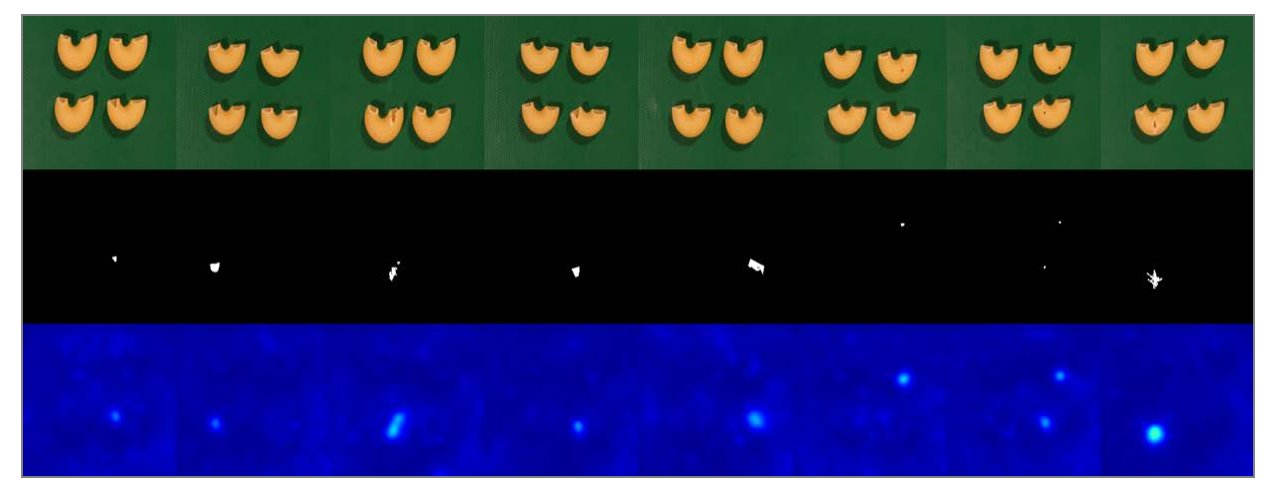}%
} 
\caption{Detection visualization on industrial datasets. In each subfigure, the three rows are origin images, ground truth anomaly masks and anomaly score maps $S_{map}$ respectively.}
\label{fig:appendix_mvtec}
\end{figure}

\end{document}